\documentclass[draftclsnofoot, 12pt, onecolumn]{IEEEtran}
\usepackage{mathrsfs}
\usepackage{bbm}
\usepackage{amsmath}
\usepackage{acronym}  
\usepackage[dvips]{color}
\usepackage{epsf}
\usepackage{times}
\usepackage{epsfig}
\usepackage{notoccite}
\usepackage{graphicx}
\usepackage{epstopdf}
\usepackage{pstricks}
\usepackage{amssymb}
\usepackage{amsthm}
\usepackage{amsxtra}
\usepackage{here}
\usepackage{rawfonts}
\usepackage{float}
\usepackage{times}
\usepackage{url}
\usepackage{cite}
\usepackage{caption}
\usepackage{subcaption}
\usepackage{blindtext}
\usepackage{enumitem}
\usepackage{xcolor,cite,etoolbox}
\usepackage{relsize}
\usepackage{siunitx}
\usepackage{lipsum}
\usepackage{graphicx}
\usepackage{tabularx}
\usepackage{xparse}
\usepackage{array}
\newcolumntype{P}[1]{>{\centering\arraybackslash}p{#1}}
\usepackage{mleftright}
\newtheorem{remark}{Remark}
\usepackage{mathtools}

\usepackage{graphics}
\usepackage{physics}
\usepackage{amssymb}
\usepackage{multicol}

\usepackage[nomain,acronym,shortcuts]{glossaries}
\makeglossaries
\newcommand*{\acro}[3][]{\newacronym[#1]{#2}{#2}{#3}}
\usepackage{datetime}
\usepackage{amssymb}
\usepackage{subcaption}
\usepackage{verbatim}

\newtheorem{corollary}{Corollary}
\newtheorem{theorem}{\bf Theorem}

\newtheorem{proposition}{\bf Proposition}
\newtheorem{lemma}{\bf Lemma}

\newtheorem{definition}{\bf Definition}

\IEEEoverridecommandlockouts
\usepackage{cite}
\usepackage{amsmath,amssymb,amsfonts}
\usepackage{algorithmic}
\usepackage{graphicx}
\usepackage{textcomp}
\usepackage{xcolor}
\usepackage{mathtools}
\usepackage{gensymb}
\usepackage{optidef}
\usepackage[linesnumbered,ruled,vlined]{algorithm2e}
\SetKwInput{KwInput}{Input}
\SetKwInput{KwOutput}{Output}
\SetKwInput{KwInitialize}{Initialize}
\usepackage{float}

\begin{document}

\acro{QoNE}{quality-of-network experience}
\acro{D2D}{device-to-device}
\acro{ToM}{theory of mind}
\acro{JEPA}{joint embedding and predictive architecture}
\acro{NN}{neural network}
\acro{KPI}{key performance indicator}
\acro{KL}{Kullback-Leibler}
\acro{NeSy}{neuro-symbolic}

\acro{HDC}{hyper-dimensional computing}
\acro{SCM}{structural causal model}
\acro{CRL}{causal representation learning}
\acro{SIR}{signal-to-interference-ratio}
\acro{SINR}{signal-to-interference-plus-noise-ratio}
\acro{PCP}{Poisson cluster process}
\acro{CoMP}{coordinated multi-point}
\acro{BS}{base station} 
\acro{MD-CoMP}{macrodiversity CoMP transmission}
\acro{MAC}{medium-access-control}
\acro{JT-CoMP}{joint transmission CoMP}
\acro{CoMP-JT}{coordinated multipoint joint transmission}
\acro{SBS}{small base station}
\acro{MDSD}{multiple devices to a single device}
\acro{MDS}{maximum distance separable}
\acro{SCN}{small cell network}
\acro{PPP}{Poisson point process}
\acro{TCP}{Thomas cluster process}
\acro{CSI}{channel state information}
\acro{PDF}{probability distribution function}
\acro{PMF}{probability mass function}
\acro{RV}{random variable}
\acro{i.i.d.}{independently and identically distributed}
\acro{MBMS}{multimedia broadcasting multicasting service}
\acro{EE}{energy efficiency}
\acro{HCP}{hard-core placement}
\acro{CCDF}{complementary cumulative distribution function}
\acro{CDF}{cumulative distribution function}
\acro{PC}{probabilistic caching}
\acro{RC}{random caching}
\acro{CPF}{caching popular files} 
\acro{PGFL}{probability generating functional}
\acro{KKT}{Karush-Kuhn-Tucker}
\acro{PGF}{point generating function}
\acro{SCA}{successive convex approximation}
\acro{FHD}{full-high-definition}
\acro{UHD}{ultra-high-definition}
\acro{VR}{virtual reality}
\acro{AR}{augmented reality}
\acro{5G}{fifth-generation}
\acro{QoS}{quality-of-service}
\acro{QoE}{quality-of-experience}
\acro{IoT}{Internet of Things}
\acro{MHCPP}{Matern hardcore point process}
\acro{LoS}{line-of-sight}
\acro{NLoS}{non-line-of-sight}
\acro{PSD}{power spectral density}
\acro{MEC}{mobile edge computing}
\acro{E2E}{end-to-end}
\acro{THz}{terahertz}
\acro{CLT}{central limit theorem}
\acro{HQ}{High Quality}
\acro{eMBB}{enhanced mobile broadband}
\acro{URLLC}{ultra reliable low latency communications}
\acro{mmWave}{millimeter wave}
\acro{EVT}{extreme value theory}
\acro{GEV}{generalized extreme value}
\acro{LIS}{large intelligent surface}
\acro{RIS}{reconfigurable intelligent surface}
\acro{RF}{radio frequency}
\acro{UE}{user equipment}
\acro{MIMO}{multiple-input multiple-output}
\acro{EVaR}{entropic value-at-risk}
\acro{DNN}{deep neural network}
\acro{MDP}{Markov decision process}
\acro{RL}{reinforcement learning}
\acro{RNN}{recurrent neural network}
\acro{ANN}{artificial neural networks}
\acro{LSTM}{long short-term memory}
\acro{ReLu}{rectified linear unit}
\acro{VaR}{value-at-risk}
\acro{SNR}{signal-to-noise ratio}
\acro{AoSA}{array of subarray}
\acro{XR}{extended reality}
\acro{AoA}{angle of arrival}
\acro{ULA}{uniform linear array}
\acro{AoD}{angle of departure}
\acro{EM}{electromagnetic}
\acro{HRLLC}{s high-rate and high-reliability low latency communications}
\acro{6DoF}{six degrees of freedom}
\acro{MR}{mixed reality}
\acro{PAPR}{peak to average power ratio}
\acro{OFDM}{orthogonal frequency-division multiplexing}
\acro{OFDMA}{orthogonal frequency-division multiple access}
\acro{SC-FDM}{single carrier frequency-division multiplexing}
\acro{ToA}{time of arrival}
\acro{MUSIC}{multiple signal classification}
\acro{IoE}{Internet of Everything}
\acro{DT}{digital twin}
\acro{PT}{physical twin}
\acro{CT}{cyber twin}
\acro{DRL}{deep reinforcement learning}
\acro{FL}{federated learning}
\acro{DL}{deep learning}
\acro{CRAS}{connected robotics and autonomous system}
\acro{CL}{continual learning}
\acro{MSE}{mean squared error}
\acro{EWC}{elastic weight consolidation}
\acro{ML}{machine learning}
\acro{GD}{gradient descent}
\acro{MLP}{multi layer perceptron}
\acro{TL}{transfer learning}
\acro{AI}{artificial intelligence}
\acro{NFT}{non fungible token}
\acro{H2A}{human-to-avatar}
\acro{A2A}{avatar-to-avatar}
\acro{UAV}{unmanned aerial vehicle}
\acro{NTN}{non-terrestrial networks}
\acro{CIS}{connected intelligence system}
\acro{QoVE}{quality-of-virtual experience}
\acro{OOD}{out-of-distribution}
\acro{RAN}{radio access network}
\acro{PVD}{physical-virtual-digital}
\acro{Tx}{transmitter}
\acro{Rx}{receiver}
\acro{RRM}{radio resource management}
\acro{AGI}{artificial general intelligence}
\acro{LLM}{large language model}
\acro{FM}{foundation model}
\acro{QoDE}{quality-of-digital experience}
\acro{QoPE}{quality-of-physical experience}
\acro{POMDP}{partially observable Markov decision process}
\acro{HD}{hyper-dimensional}
\acro{IIT}{integrated information theory}
\acro{DAG}{directed acyclic graph}
\acro{GNN}{graph neural network}
\acro{XAI}{explainable learning}
\acro{PFC}{pre-frontal cortex}
\acro{BG}{basal ganglia}
\acro{VFE}{variational free energy}
\acro{EFE}{expected free energy}
\acro{MFA}{mean field approximation}
\acro{VMP}{variational message passing}
\acro{NESS}{non-equilibrium steady state}

\title{\centering Active Inference as the Test-Time Scaling Law for Physical AI Agents
\thanks{O. Hashash and W. Saad are with the Bradley Department of Electrical and Computer Engineering and the Institute for Advanced Computing, Virginia Tech, Alexandria, VA, USA. Emails: \protect{omarnh@vt.edu, walids@vt.edu.}}
\thanks{C.K. Thomas is with the Department of Electrical and Computer Engineering, Worcester Polytechnic Institute, Worcester, MA, USA. Email: \protect{cthomas2@wpi.edu.}}
\thanks{M. Debbah is with is with Khalifa University of Science and Technology, Abu Dhabi 127788, United Arab Emirates, and also with the CentraleSupelec, University Paris Saclay, 91192 Gif-sur-Yvette, France. E-mail: \protect{merouane.debbah@ku.ac.ae}.}
\thanks{K. Friston is with the Queen Square Institute of Neurology, University College London, London, UK. Email: \protect{k.friston@ucl.ac.uk}}
\thanks{A. Razi is with the Turner Institute for Brain and Mental Health at School of Psychological Sciences and Monash Biomedical Imaging, Monash University, Clayton, Australia and CIFAR Global Scholars Program, Toronto, Canada and Queen Square Institute of Neurology, University College London, London, UK. Email: adeel.razi@monash.edu.}}%

\author{\normalsize Omar~Hashash,~\IEEEmembership{\normalsize Member,~IEEE,} Christo~Kurisummoottil~Thomas,~\IEEEmembership{\normalsize Senior~Member,~IEEE,}
Walid~Saad,~\IEEEmembership{\normalsize Fellow,~IEEE,}
M{\'e}rouane~Debbah,~\IEEEmembership{\normalsize Fellow,~IEEE,}
Karl~Friston,
and~Adeel~Razi,~\IEEEmembership{\normalsize Senior~Member,~IEEE} 

\vspace{-1cm}}%
\maketitle

\vspace{-0.9cm}

\begin{abstract}
In this paper, a novel \emph{test-time scaling law for physical artificial intelligence (AI) agents} is introduced. This scaling law enables physical AI agents to \emph{reason with their world models so as to generalize in unforeseen scenarios} that appear at test time. In particular, the derived scaling law is grounded in the first principle of \emph{active inference} that equips physical AI agents with the general objective to \emph{survive} in the real world under which their specific, narrow task objectives are subsumed. Active inference achieves this by providing the reasoning necessary to resolve the \emph{prediction errors} that arise when the agent encounters unforeseen situations outside its training distribution, enabling physical AI agents to generalize in new situations that appear in non-stationary environments.
The proposed scaling law captures this generalization by dynamically updating the physical AI agent's policy with this reasoning capability at test time. This policy update is then modeled as a soft Bayesian inference process in which the beliefs about the policy are updated using the reasoning that reduces the prediction errors (i.e., surprise) expected under allowable policies. Notably, the resulting posterior policy admits a biologically plausible interpretation which recovers the scaling mechanism that engages the brain's basal ganglia and prefrontal cortex at test time. To solve this analytically intractable inference problem, a variational inference solution that minimizes free energy bounds is developed to reduce prediction error, and the proposed framework is further extended to enable \emph{learning} beyond training by reinforcing new instances, resolved at test time, in both the policy and world model. This approach is shown to enable physical AI agents to \emph{continuously} learn through real-world deployment. Unlike existing scaling laws that remain constrained by factors such as model size and training data, the derived solution ultimately \emph{scales with the continuous experience} of a physical AI agent in the real world. Simulation results in the context of an autonomous driving task demonstrate that the proposed solution outperforms model-free Q-learning and model-based Bayesian reinforcement learning, by achieving robust generalization to unforeseen scenarios while improving inference efficiency by over $36\%$.
\vspace{-0.4cm}
\end{abstract}


\vspace{-0.4cm}

\section{Introduction}
\indent Physical \ac{AI} agents such as humanoid robots and autonomous vehicles inevitably encounter \emph{unforeseen}\footnote{An unforeseen scenario refers to an unexpected, unanticipated, or unfamiliar state in which the \ac{AI} system has limited or no prior experience to handle properly.} situations when operating in the physical world~\cite{saad2025artificial}. This is due to the dynamic, non-stationary nature of the real world and its infinite state space that prevents the agent from fully experiencing it. In consequence, the experience of a physical \ac{AI} agent about the world remains limited to its training phase. This, in turn, explains why physical \ac{AI} agents often fail when they experience new, unforeseen scenarios \emph{at test time}.
Recent real-world examples of such failures include the December 2025 Waymo power outage incident in San Francisco, where robotaxis encountered non-functioning traffic lights -- an unforeseen scenario that caused multiple vehicles to stop abruptly mid-intersection, forcing Waymo to suspend its service entirely~\cite{ding2025waymo}.

To resolve the situations that unfold beyond their training domain, physical \ac{AI} agents must then rely on their \emph{reasoning} at test time. This reasoning about the world can be facilitated through \emph{world models} which are an essential cognitive component that bring forth an understanding of the physical world in terms of its real-time state, causal structures, and dynamic evolution~\cite{ha2018world}. Accordingly, reasoning typically leverages world models to simulate counterfactual scenarios and \emph{plan} future states of the world while optimizing the cost of the actions\footnote{In the field of psychology, this is also known as System 2 inference~\cite{kahneman2011thinking}.}~\cite{LeCun2022OR}. Ultimately, the goal of this reasoning is to drive the \ac{AI} agent to generalize in these new situations. However, while reasoning indeed paves the way for potentially optimizing the actions of the \ac{AI} agent in unforeseen scenarios, it does not inherently account for generalization. Such generalization in the physical world would still require merging this reasoning about the world\footnote{This is often known as common-sense reasoning.} with the narrow task knowledge of the \ac{AI} agent.
In other words, achieving such generalization requires a solution that expands the accumulated experience of the physical \ac{AI} agent with these reasoning capabilities to optimally handle unforeseen scenarios. If achieved, such a solution can lead to a \emph{test-time scaling law for physical \ac{AI} agents} that need to generalize in unforeseen scenarios, analogous to~\ac{LLM} agents that must reason to improve their response to harder prompts. \emph{Therefore, this paper seeks to uncover how reasoning with world models can enable physical \ac{AI} agents to generalize in unforeseen scenarios, while demonstrating how such a solution, when derived from first principles, can ultimately serve as their test-time scaling law.}


The first step to devise a solution to the physical \ac{AI} generalization problem is to draw a rigorous analogy with humans who generalize robustly and intuitively in unforeseen scenarios. Humans possess a remarkable ability to detect unforeseen scenarios in the world, which often triggers a natural response known as \emph{surprise}. In particular, this surprise reflects the mismatch between the world as \emph{perceived} and its expected state predicted under one's internal world model~\cite{rao1999predictive}. Indeed, humans anticipate their sensory input and any corresponding \emph{prediction error} signals the presence of an unforeseen scenario that deviates from their accumulated experience, prompting them to engage in deliberate reasoning to resolve their surprise. For humans, resolving prediction errors thus implies restoring consistency between predictions under their internal world model and input from the world.
Humans typically achieve this by transitioning back to their \emph{preferred} states, from which they can predict prospective transitions and recover their desired sensory input. Clearly, in this situation, reasoning plays a key role in enabling humans to  \emph{actively} drive the physical world closer to their preferences. In other words, here, the purpose of reasoning is to take actions that can eliminate the source of surprise, i.e., ``the prediction error". The mechanism that formalizes this key process is called \emph{active inference}~\cite{friston2017active}, which models human perception as an inference process shaped by such deliberate reasoning and actions.
Thus far, the framework of active inference has been used to study how humans detect and reason about unforeseen scenarios through a control of their prediction errors. More broadly, active inference provides a first-principles account of how all living systems must exhibit this form of reasoning to \emph{survive} in the real world and ensure their existence~\cite{goldbeter2018dissipative}. \emph{Remarkably, the absence of this form of ``biological" reasoning at test time in physical \ac{AI} agents, to date, might possibly explain their failures and limited capacity to survive in the real world~\cite{de2026active}.}

\begin{figure}
	\centering
	\includegraphics[width=\columnwidth]{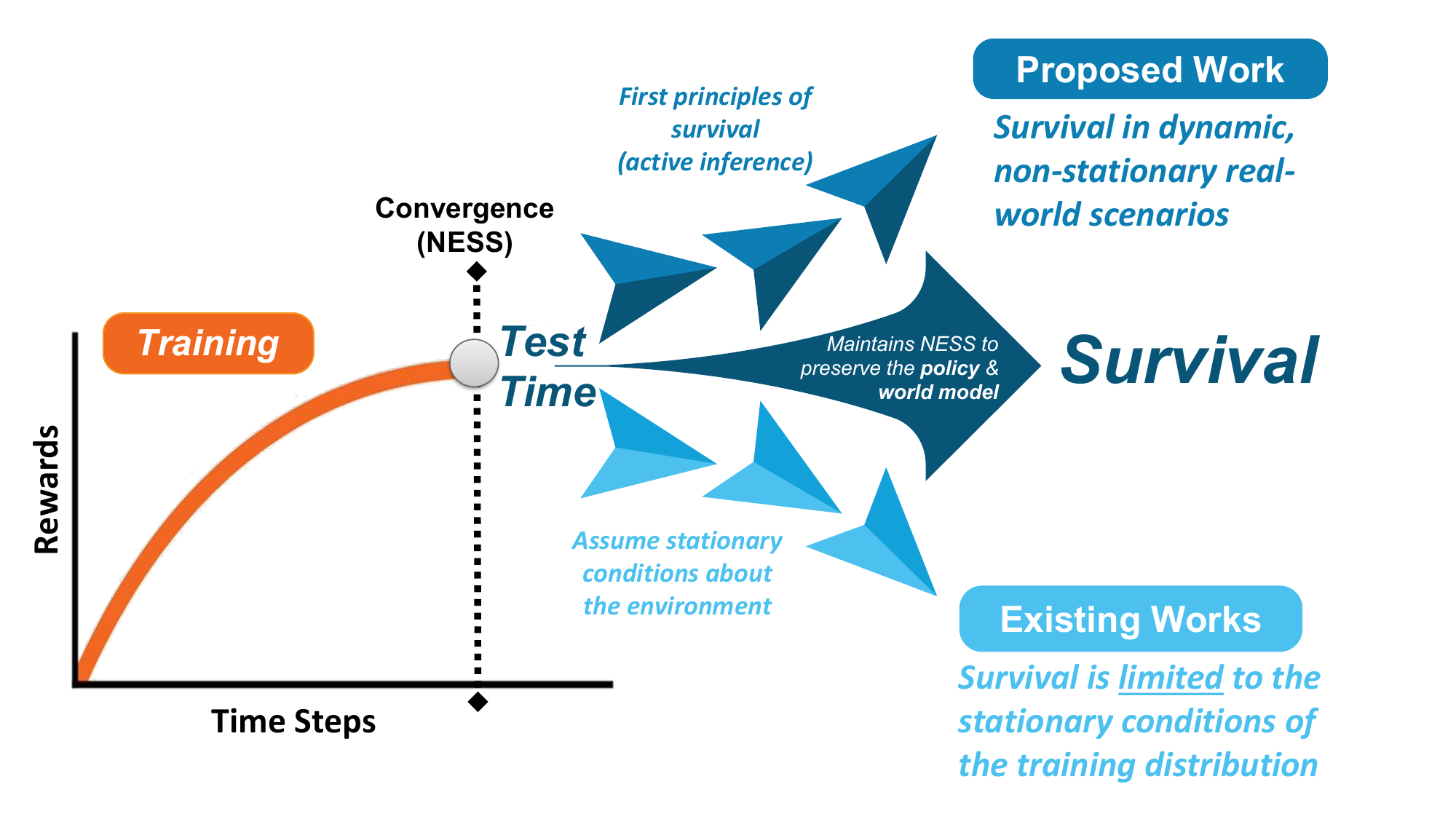}
	\caption{\small{Illustration of the solutions for survival at test time by conserving the NESS of physical \ac{AI} agents, according to the different underlying conditions assumed about the world.}}
	\label{Survival}
	\vspace{-0.50cm}
\end{figure}

Indeed, the physical world is full of unforeseen scenarios that an \ac{AI} agent must deal with to survive. By definition, survival then becomes one of the main facets of an \ac{AI} agent that generalizes in the world. From a physics standpoint, agentic systems that survive by persisting in the world are those that evolve to achieve a \emph{\ac{NESS}} with their environment\footnote{\ac{NESS} is defined by two factors. First, the non-equilibrium considers that the agent continues to interact with its environment. Second, the steady state implies that the probability distribution of the states of the agent remains constant over time.}~\cite{friston2013life}. This is similar to what today's \ac{AI} systems~(e.g., \ac{RL}) reach when their policies and world models converge to optimal behavior at the end of their training phase, as shown in Fig.~\ref{Survival}.
Hence, \ac{AI} agents must strive to preserve the functionality of these policies and world models to \emph{survive at test time}. Here, Fig.~\ref{Survival} shows the two ways in which this can happen.
Current physical \ac{AI} solutions such as \ac{RL} achieve this by assuming stationary conditions about the world at test time. However, these assumptions are often violated in practice due to the uncertain and dynamic nature of the real world\footnote{These assumptions are meaningfully valid if the environment stays the same as in closed-world scenarios (e.g., chess).}~\cite{dulac2021challenges}. This necessarily renders policies (and world models) sub-optimal, setting up these physical \ac{AI} systems to fail at test time. However, by considering the alternative approach which equips \ac{AI} agents with the first principle that ensures survival at test time (i.e., active inference), we can ensure that physical \ac{AI} agents continue to function even in non-stationary and uncertain environments.
This helps meet the requirements of real-world physical environments, where survival cannot be guaranteed by simply relying on stationarity assumptions.
In fact, the urge to survive will drive \ac{AI} agents to generalize in the face of unforeseen scenarios that they encounter. \emph{Remarkably, this fundamental solution remains underexplored in the realm of \ac{AI}, thus making generalization at test time one of the central open challenges in physical \ac{AI}.}







\subsection{Prior Works}
Generalizing at test time has motivated a growing body of work and emerging architectures~\cite{openai_gpt5, assran2025v, yang2025mindjourney, pi07} that position reasoning and world models as their key pillars. 
OpenAI's \text{GPT-5}~\cite{openai_gpt5} employs dynamic routing to selectively invoke chain-of-thought reasoning at test time, improving performance and generalizing when handling complex or unfamiliar tasks. Meta's \text{V-JEPA}~2~\cite{assran2025v} constructs a world model through self-supervised video pretraining with minimal robot interaction, which enables zero-shot generalization and planning in novel environments without task-specific rewards. 
In~\cite{yang2025mindjourney}, a vision-language model is coupled with a world model that enables reasoning over imagined egocentric views for generalizing to unforeseen spatial queries without fine-tuning. 
Physical Intelligence's $\pi_{0.7}$~\cite{pi07} extends this further to physical \ac{AI} by conditioning a generalist vision-language-action model on sub-goal images generated by a world model, thereby allowing it to reason over enriched visual context before predicting its actions to generalize across unseen environments and tasks in a compositional manner.

Despite the significant improvement in generalization, the approaches in~\cite{openai_gpt5, assran2025v, yang2025mindjourney, pi07} remain limited from complementary viewpoints. While GPT-5's dynamic router bears a conceptual resemblance to a prediction error signal, i.e., selectively triggering additional reasoning when a query exceeds the \ac{LLM}'s ability, it lacks a grounded, embodied world model from which such a signal can be derived in a principled manner. This distinction is crucial, because a test-time \emph{scaling law must emerge from first principles of intelligent behavior}, rather than from a learned routing heuristic.
Although \text{V-JEPA~2}~\cite{assran2025v} (and its extensions~\cite{murlabadia2026vjepa2_1}) generalizes across unseen environments and objects, it relies on costly deliberative planning to reason through each step before acting, at the expense of efficiency in policy execution in performing tasks. This is in contrast to GPT-5 which \emph{switches} between efficient execution and deliberative reasoning modes. 
In both~\cite{yang2025mindjourney} and \cite{pi07}, \emph{reasoning is mediated through language}, with the world model serving merely as a downstream visual generator, rather than as the substrate through which the agent directly reasons about world states and plans action consequences. 
Such reasoning is constrained to the statistical patterns in the training data that works for compositional generalization, but falls short in driving physical \ac{AI} agents such as $\pi_{0.7}$ to truly generalize in the world. This is evident as $\pi_{0.7}$ successfully generalizes to novel task combinations that can be decomposed from tasks seen during training, while failing when faced with scenarios outside its training domain. To resolve these new situations at test time, $\pi_{0.7}$ resorts to additional human-provided reasoning to fill in the missing gap. This limitation clearly highlights the need for the mechanism that autonomously generates this reasoning about its actions (i.e., via a world model) when confronting truly unforeseen scenarios. Furthermore, none of these approaches~\cite{openai_gpt5, assran2025v, yang2025mindjourney, pi07} addresses the \emph{continual adaptation} of the policy and world model over time, leaving the \ac{AI} agent unable to incorporate new experiences that it generalizes to at test time. 

From the limitations of~\cite{openai_gpt5} and \cite{ assran2025v}, it becomes clear that there is a need for a principled method that detects prediction errors at test time to trigger the transition from efficient policy execution to deliberative reasoning. Addressing this gap defines the first step toward designing a physical \ac{AI} agent that can truly generalize in unforeseen scenarios. Unlike \cite{yang2025mindjourney} and \cite{pi07}, reasoning with the world model must aim to optimize the agent's actions in response to the prediction error, instead of providing additional visual input. While active inference provides a solution for detecting such prediction errors and reasoning about how to resolve them, \emph{it is still necessary to design a first-principle scaling mechanism that enables the \ac{AI} agent to update its policy with this reasoning at test time.} 
From a biomimetic perspective in which both \ac{RL} and active inference engage dopaminergic signaling in the brain\footnote{A dopaminergic signaling system encompasses a neuromodulatory mechanism in which dopamine encodes reward prediction errors and broadcasts them to regulate synaptic plasticity and action selection throughout the brain. This signaling is exploited in both \ac{RL} and active inference.}, revisiting the neuronal basis of reasoning and planning offers a natural grounding for a rigorous test-time scaling law. \emph{This intersection between learning and inference further motivates the development of a unified framework that blends both training and test time paradigms and can serve as a bedrock for physical \ac{AI} agents that need to continuously reason, generalize, and learn in the physical world.}

\vspace{-0.2cm}

\subsection{Contributions}

The main contribution of this paper is a novel \emph{test-time scaling law for physical \ac{AI} agents} grounded in \emph{active inference}. This scaling law considers how physical \ac{AI} agents reason with their world models to resolve the prediction errors that arise in unforeseen scenarios. This, in turn, drives \ac{AI} agents to dynamically scale their policies with this reasoning to \emph{generalize at test time}. Inspired by the fact that every living system evinces this reasoning capability to ensure its survival~\cite{kirchhoff2018markov}, our test-time scaling law identifies the \emph{minimum amount of generalization} that any physical \ac{AI} agent must exhibit to successfully achieve their narrow task objectives in the real world. Ultimately, this provides a solution to the profound challenge facing physical \ac{AI} (e.g.,~\ac{RL}) agents that fail to achieve their objectives in \emph{non-stationary} environments. Furthermore, the proposed solution naturally extends the learning framework beyond training by \emph{incorporating their new experience at test time}. As such, the proposed approach enables the AI agents to \emph{reinforce} the new instances encountered and resolved at test time, as they update their policy and world model to reduce prediction errors that would otherwise result from similar instances in the future. Thereby, the derived scaling law relaxes the stationary assumptions that limit \ac{RL} at test time, enabling agents to incorporate new knowledge and accumulate experience continuously through interaction with the world. In summary, our key contributions include:

\begin{itemize}
\item Starting from active inference as the first-principles account of self-organization in open systems, we derive a test-time scaling law for physical \ac{AI} agents. This scaling law verifies and grounds the empirical test-time scaling observed in \acp{LLM} (e.g., GPT-5), while extending it to embodied physical \ac{AI} agents operating in non-stationary environments (see Fig.~\ref{Test_Time_Scaling}). This verifies that it is policies that naturally scale with reasoning at test time and \acp{LLM} scale as a simplified, specific case of the original solution. 
The derived scaling law posits \emph{survival} as the universal objective common to all intelligent agents, under which their narrow task objectives are subsumed and governed in the real world. This transforms physical \ac{AI} agents into dual-objective systems that aim to maximize a \emph{general survival reward} formalized by a prediction error through active inference and a narrow task reward (e.g., driving, walking, etc.) instilled through techniques such as \ac{RL}.

\item By associating the dopaminergic signaling of \ac{RL} with its role in active inference, we provide the \emph{first computational model that bridges the role of dopamine during training~\cite{schultz1997neural} to its role at test time within a single mathematical formulation}.
We find that, after convergence of the policy during training (where dopamine converges to its baseline), dopamine spikes again at test time when prediction errors arise in non-stationary environments to initiate the transition from policy execution to reasoning. Neurobiologically, this reasoning is instantiated by neural signals from the \ac{PFC} to the \ac{BG} which speaks to how the value function of the policy is modulated with reasoning in response to prediction errors at test time. This modulation provides the missing test-time scaling mechanism of the policy, that can be mathematically translated into a soft Bayesian belief update with the virtual evidence acquired via reasoning to provide the posterior \emph{scaled policy}.
Remarkably, this Bayesian update decomposes action policies into feed-forward and inference terms, where the feed-forward term recovers the original \ac{RL} policy and the inference term, \emph{which is lacking in existing solutions like~\cite{openai_gpt5, assran2025v, yang2025mindjourney, pi07}}, provides the missing modulation that the \ac{AI} agent needs to swiftly adapt its policy when prediction errors arise in non-stationary environments.

\item As the inference problem underlying the resolution of prediction errors and policy scaling is intractable, we resort to a variational inference solution that introduces the notions of  \emph{\ac{VFE}} and \emph{\ac{EFE}} as variational objectives that must be minimized to resolve prediction errors. Modeling the resolution of prediction errors as a gradient descent on the free energy landscape naturally gives the scaled policy as the solution that simultaneously satisfies both the general survival objective and the narrow task objective, enabling the agent to generalize to unforeseen scenarios at test time. 

\begin{figure}
	\centering
	\includegraphics[width=\columnwidth]{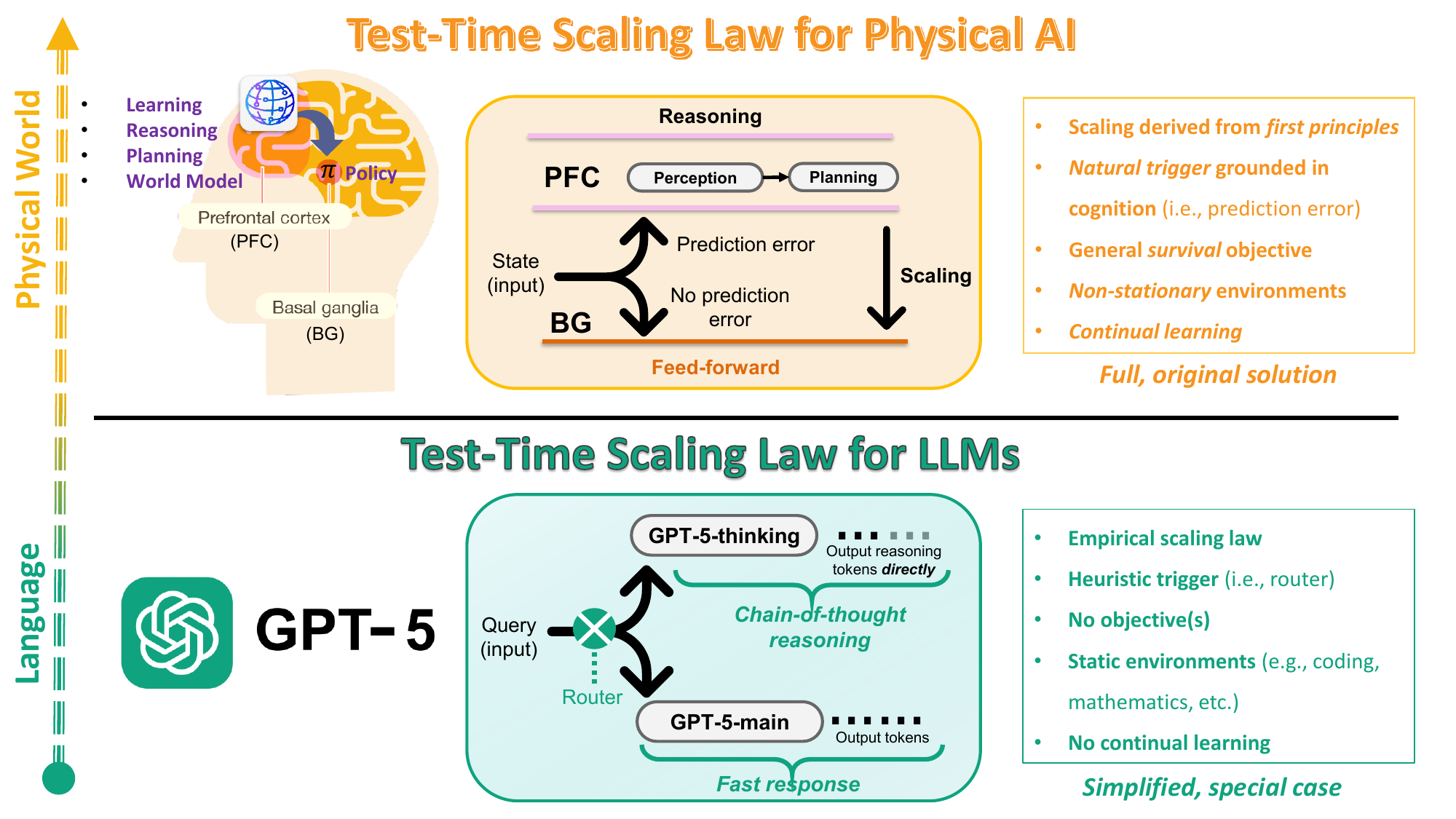}
	\caption{\small{Test-time scaling law for physical \ac{AI} agents (top) vs. test-time scaling law in \acp{LLM} (bottom). The proposed framework grounds test-time scaling in active inference, where prediction errors trigger the transition from \ac{BG} execution to \ac{PFC} reasoning, followed by a feedback to scale the policy in the \ac{BG}. Here, the \ac{LLM}'s (e.g.,~GPT-5) routing and scaling represent a special empirical case of the proposed framework.}}
	\label{Test_Time_Scaling}
	\vspace{-0.250cm}
\end{figure}

\item To reduce future prediction errors, as a means to increase chances of survival, the agent must update its world model to anticipate the \emph{resolved}, unforeseen scenarios. We model this learning process as a Bayesian belief update that reinforces the new experiences accumulated at test time into both the policy and world model. This enables physical \ac{AI} agents to \emph{continuously learn from their own experience} and \emph{autonomously self improve through real-world deployment}, revealing that learning and inference form a continuum driven by prediction errors. Unlike \acp{LLM} that mainly rely on scaling their parameters and training data to improve performance, \emph{our proposed solution scales with the experience of the agent directly through continuous interaction with the world}, transcending the limits of scaling laws that are constrained by data and parameters~\cite{hoffmann2022training, kaplan2020scaling, snell2024scaling}.

\item Simulation results for an autonomous driving task demonstrate robust generalization to unforeseen scenarios while improving inference efficiency by more than $36$\% compared with model-free Q-learning and model-based Bayesian \ac{RL}.

\end{itemize}

\subsection{Key Insights}
The key findings from our derived scaling law demonstrate that the restrictive stationarity assumption underlying \ac{RL} can be relaxed at test time, provided that agents are equipped with the reasoning necessary to eliminate prediction errors. This is possible because the prediction errors associated with the general survival reward operate alongside the narrow task reward throughout training, and, yet, in the existing state of art, they remain unaccounted for at test time. 
This implies that prediction errors were implicitly being minimized while trying to maximize the narrow task reward during training.
Particularly, when an \ac{AI} agent thus fails in a non-stationary environment, it is effectively losing control of its general survival reward which exposes the agent to existential risk, before degrading its narrow task performance. In other words, maximizing long-term reward in conventional \ac{RL} was implicitly tied to survival, and that restoring the general reward through active inference, at test time, as proposed in this paper, addresses a limitation that has been largely ignored in real-world deployments of physical \ac{AI} solutions. Our approach unifies \ac{RL} and active inference as manifestations of the same principle at different levels, where \ac{RL} implicitly minimizes prediction errors within the training distribution of the narrow task objective and active inference makes this minimization explicit when the training distribution is violated (e.g., through a distribution shift~\cite{scholkopf2021toward}) at test time.

Intuitively, the derived solution leverages the efficiency of RL and the context-sensitivity of active inference by integrating both within the same framework. In other words, agents learn a world model (and policy) under the assumption there exists a single-state action policy. From the perspective of active inference, this is a process of habit-learning; namely, ``this is what I do, given this state of the world". However, when the context switches from training to test time, one can now engage active inference to evaluate alternative policies \textit{when, and only when, there is evidence that the context has changed}. This evidence is gathered by observing an increase in surprise (i.e., prediction error), and it can also be used to weight subsequent updates to the state-action policy (encoded in Q for instance). We can also view this as a form of Bayesian model averaging but applied to \ac{RL} (e.g.,~Q-learning). In short, if there is no surprise, then there is no evidence that the context has changed and there is no policy update.

\subsection{Organization}
The rest of the paper is organized as follows. Section~\ref{system model} presents the system design and proposed active inference framework that allows establishing the test time scaling law for physical \ac{AI}. Section~\ref{Perception} provides a variational inference solution that transforms the intractable inference which captures perception into a \ac{VFE} minimization problem. Section~\ref{planning} transforms the intractable inference problems of planning and action into an \ac{EFE} minimization problem. Section~\ref{Learning via Inference} shows how updating the world model (and policy) by learning the unforeseen scenarios, when done at test time, can be modeled as an inference problem that reinforces the new experiences. Section VI presents the simulation results and analysis for an autonomous driving task with an unforeseen jaywalking scenario encountered at test time. Finally, Section VII concludes the paper.

\section{Test-Time Scaling Law for Physical \ac{AI} Agents: Proposed System Design and Active Inference Framework}
\label{system model}
Consider a geographical zone representing a region of the real world, as shown in Fig.~\ref{System_Model}. 
This zone comprises autonomous physical \ac{AI} agents and a wireless network infrastructure that provides connectivity and computing services for these agents\footnote{The proposed system can consider various alternative designs. Here, we assumed that a wireless network, equipped with sufficient resources (e.g., compute, throughput, latency, etc.) is available to service the world model needs of the agents. The implementation challenges and details of  such an architecture are discussed in our prior work~\cite{saad2025artificial}. Alternatively, one can also deploy the world model directly on the agent without the use of a wireless network. The proposed active inference framework and generalization approach, as well as our results, are independent of this design choice.}.
In particular, the wireless network \ac{BS} is equipped with an edge computing server to provide world model services for the physical \ac{AI} agents in this zone. This world model enables the \ac{AI} agent to detect prediction errors that signal an unforeseen scenario and counterfactually reason about the consequences of its actions to resolve these errors. In other words, the \ac{AI} agents offload their reasoning process to the network's edge server.
To enable this process, each \ac{AI} agent will have \acp{DT} alongside the world model at the network edge, that are used to model the alternative policies that a physical \ac{AI} agent may undertake. Unlike conventional \acp{DT} that serve as a simulation and replication tool, the goal of \acp{DT} in our framework is to operate the world model at the network edge and return a real-time reasoning feedback to the physical \ac{AI} agent, i.e.,~\ac{PT}, to scale its policy at test time and generalize in an unforeseen scenario~\cite{thomas2026passive}.

For example and without loss of generality, the physical \ac{AI} agent in our framework could be an autonomous vehicle that has been trained via \ac{RL} to execute real-time driving decisions\footnote{Although the physical \ac{AI} agent in this work is considered to be an autonomous vehicle for illustration purposes, this framework generally applies to any other real-world agent (e.g., robot, drone, etc.) as well.}, as illustrated in Fig.~\ref{System_Model}. This agent follows a stochastic policy $\pi_{o}(a \mid s)$, which specifies a distribution over driving actions $a \in \mathcal{A}$ conditioned on world states $s \in \mathcal{S}$.
As the \ac{AI} agent may encounter unforeseen scenarios while navigating the world (e.g., a jaywalking pedestrian crossing the road on a green traffic light), it is crucial that the network leverages its \acp{DT} and provides the agent with the necessary reasoning capabilities to deal with such situations. This reasoning is shared from the network to the \ac{AI} agent through a set of configurations $\boldsymbol{u}\in \mathcal{U}$ that encode the general reward stemming from the actions that favor survival, along with the prediction error. The agent can then leverage these configurations to update its policy from $\pi_o$ to $\pi_{o}'$ through a test-time scaling mechanism. Thereby, the agent will now choose its action from the scaled policy $\pi_{o}'$ which generalizes the agent's experience in the unforeseen scenario appearing at test time instant $t$.

As shown in Fig.~\ref{System_Model}, the physical \ac{AI} agent interacts with a set $\mathcal{N}$ of $N$ physical assets available in the world. Examples of these assets may include elements such as the pedestrians, vehicles, and traffic lights, that comprise the physical world. In contrast to the \ac{AI} agent that receives reasoning feedback through its \acp{DT} to scale its policy, these physical assets take part in the world model without incurring any feedback to their physical counterparts. This is evidently because only physical \ac{AI} agents are the intelligent systems that may require test-time reasoning to generalize in unforeseen scenarios. To initiate this test-time scaling, the first step is to build the \emph{world model} over the network, which we explain next.

\subsection{World Model}

The world model over the network captures the interactions between all the elements (i.e., \ac{AI} agent and physical assets) in the considered geographical zone. Henceforth, these interacting elements compose an environment whose components are causally dependent at a given time index $\tau \in \{1, \dots, T\}$, where $T$ is the total number of time instances~\cite{richens2024robust}. These causal dependencies among the states of the elements can be represented in the form of a structured prior, as follows:
\begin{equation}
p(s_{\tau}) = p(s^0_{\tau}) p(s^1_{\tau} \mid s^0_{\tau}) p(s^2_{\tau} \mid s^0_{\tau}, s^1_{\tau}) \dots p(s^N_{\tau} \mid s^0_{\tau}, \dots, s^{N-1}_{\tau})  = p(s^0_{\tau}) \prod_{n=1}^{N} p(s^n_{\tau} \mid s^0_{\tau}, \dots, s^{n-1}_{\tau}),
\end{equation}
where  \( s_{\tau} \triangleq (s^0_{\tau}, s^1_{\tau}, \dots, s^N_{\tau}) \), with \(s^0_{\tau}\) being the state of the PT at time $\tau$, and \( s^n_{\tau} \) $\forall n \in \mathcal{N}$ being the state of each physical asset $n$ at time $\tau$.

\begin{figure}
	\centering
	\includegraphics[width=\columnwidth]{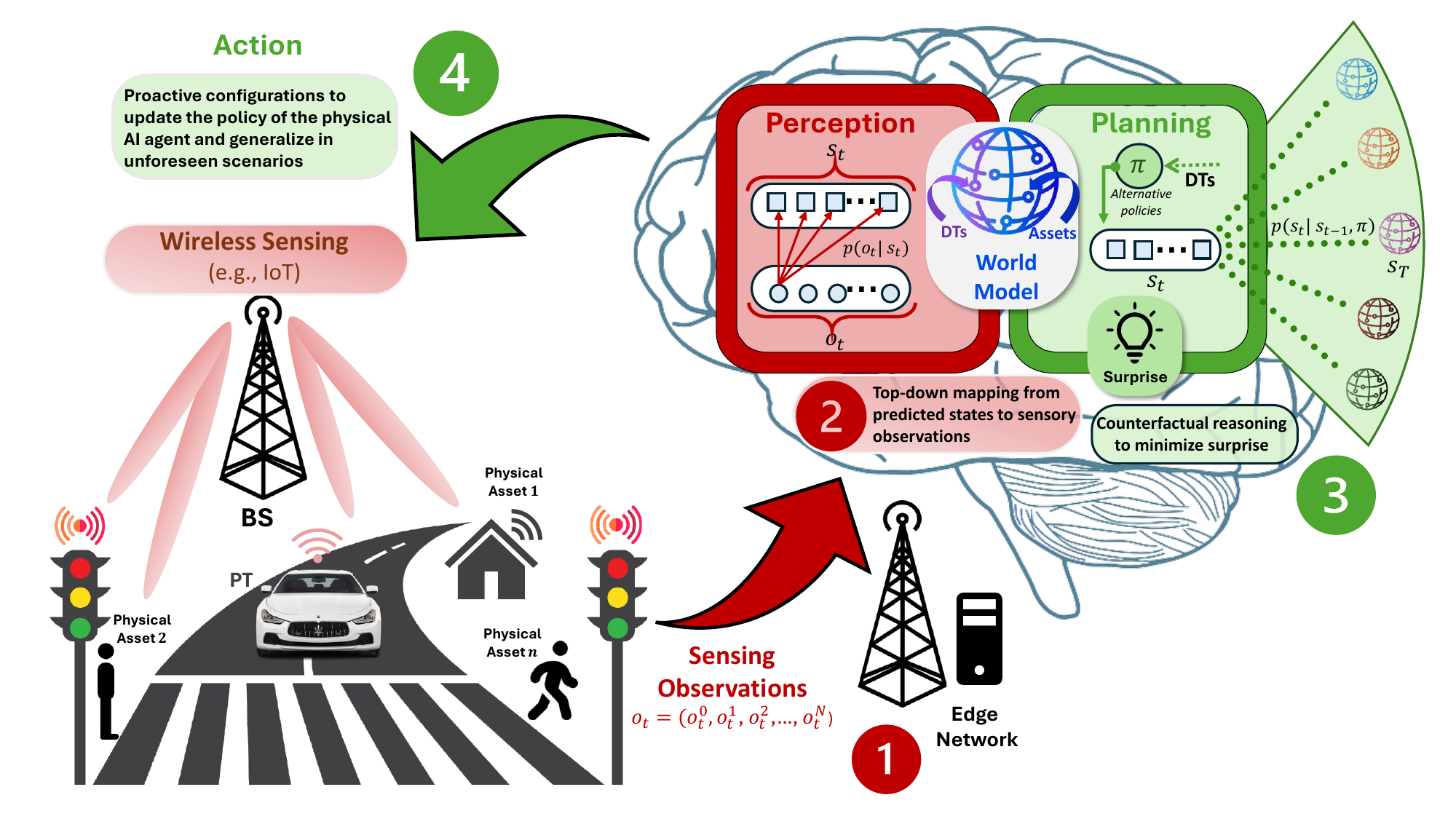}
	\caption{\small{Illustration of our test-time scaling framework comprising a physical \ac{AI} agent (e.g.,~an autonomous vehicle) that reasons through the world model at the wireless network edge to generalize in unforeseen scenarios.}}
	\label{System_Model}
	\vspace{-0.50cm}
\end{figure}

Furthermore, the network captures the state of the world $s_{\tau}$ through sensory observations \(o_{\tau} \triangleq (o^0_{\tau}, o^1_{\tau}, \dots, o^N_{\tau}) \in \mathcal{O}\) about the real-world elements. Here,  \(o^0_{\tau}\) is the sensory observation from the PT at time $\tau$ and \(o^n_{\tau}\), $\forall n \in \mathcal{N}$ is the sensory observation from each physical asset $n$ at time $\tau$. Accordingly, each state $s_{\tau}^{n}$ generates a corresponding sensing observation \( o^n_{\tau} \) $\forall n \in \mathcal{N}$. In general,  sensing observations can be captured through different wireless sensing technologies such as \ac{IoT} sensors. Here, these sensory observations from different elements are conditionally independent when conditioned on their states.
Accordingly, this yields a likelihood function that factorizes as: $p(o_{\tau} \mid s_{\tau}) = \prod_{i=0}^{N} p(o^i_{\tau} \mid s^i_{\tau})$. 

As the world evolves according to the causal structure of the environment, the states of the elements at time $\tau$ are causally dependent on their states in the previous time instant $\tau-1 $ and the corresponding courses of actions of the \ac{AI} agent, stemming from its policy, that can represented as $\pi = (a_{1}, \ldots, a_{T})$. Accordingly, we can now define the world model that captures the temporal and causal dependencies.

\begin{definition}
   The generative \emph{world model} $\mathcal{W}$ over the states and observations of the real-world elements is defined as:
    \begin{align}
    \mathcal{W}(o_{1:T}, s_{1:T}, \pi) \coloneqq p(o_{1:T}, s_{1:T}, \pi) &=   p(\pi) p(s_{1}) \prod_{\tau = 2} ^{T} p(s_\tau \mid s_{\tau-1},\pi) \prod_{\tau =1}^{T} p(o_\tau \mid s_\tau)   \label{world model}  
    \end{align}
\end{definition}

The world model $\mathcal{W}$ in~\eqref{world model} is a (forward) generative model that factorizes the joint distribution over observations $o_{1:T}$, states $s_{1:T}$, and actions $\pi$ into three components. The first captures the prior over the agent's policy $p(\pi)$ and the initial state distribution $p(s_1)$, the second describes the causal state transition dynamics $p(s_\tau | s_{\tau-1}, \pi)$ which encode how the world evolves over time as a consequence of the \ac{AI} agent's actions, and the third describes the observation model $p(o_\tau | s_\tau)$ which captures how the network perceives the resulting world states through its sensors. The factorization in~\eqref{world model} encodes the full causal structure of the agent's interaction with the environment, enabling it to predict future observations and plan actions accordingly.

After acquiring this world model, the network can now provide the \ac{AI} agent with the reasoning capabilities necessary to deal with unforeseen scenarios. Initially, the sensory input received in an unforeseen scenario violates the sensory observations predicted by the world model over the network. This initiates a prediction error or mismatch called \emph{surprise}~\cite{itti2009bayesian, friston2010free}. This formal notion of surprise (a.k.a., surprisal or self-information) signals the level to which the current state of the world is unpredictable under the world model $\mathcal{W}$ for which the \ac{AI} agent (i.e., policy $\pi_o$) exists. In particular, the surprise under $\mathcal{W}$ at time $\tau$ is defined by the network with respect to its received sensory observations as follows~\cite{friston2010free}: 
\begin{align}
    \mathscr{S} (\tau, \pi_o)  &= - \ln p(o_\tau \mid \pi_o) =   - \ln  \sum_{s_\tau} p(o_\tau, s_{\tau} \mid \pi_o) = - \ln \sum_{s_\tau} p(o_\tau \mid s_\tau) p(s_\tau \mid \pi_o).\nonumber
\end{align}

In essence, once $\mathscr{S} (t, \pi_o)$ exceeds a surprise threshold\footnote{In general, the surprise threshold is an adaptive parameter. However, it is considered to be a fixed parameter in the scope of this work for simplicity.} $\epsilon$, the situation is considered to be unforeseen, for which relying on $\pi_o$ by itself will not be able to generate the optimal action. In other words, an unforeseen scenario is one in which the agent's accumulated experience encoded in $\pi_o$ is insufficient to predict and respond to the current state of the world, which renders the policy $\pi_o$ sub-optimal and requires deliberative reasoning to re-evaluate options. The goal of this reasoning under the world model is to control surprise and reduce it back below the threshold~$\epsilon$. To achieve this, it is necessary that the network starts by \emph{perceiving} the world after its predictions of this world have been violated\footnote{For simplicity, we limit perception to the case in which prediction errors exceed the threshold. However, it is essential to clarify that perception is always performed in biological systems even in the absence of prediction errors. In other words, perception (as perceptual inference) is neglected when prediction errors remain below the threshold $\epsilon$.}. Then, the network engages in \emph{planning} to evaluate the actions that may reduce future (expected) surprise. By using \emph{active inference}, we precisely model these perception and planning functions as Bayesian inference processes. Hence, we need to extend the inference process to model the test-time scaling of the policy from $\pi_o$ to $\pi_o'$. The agent then \emph{acts} according to the policy $\pi_o'$ in order to resolve the uncertainty induced by an unforeseen scenario. Next, we explain how perception, planning, and action can be modeled through Bayesian inference.

\subsection{Perception, Planning, and Action as Inference}
\emph{Perception} entails aligning the world model with the current state of the world.
This is equivalent to acquiring an accurate estimate of the state of the world at time $t$ over the network.
This estimate can be conceptualized as the process of inferring the most likely hidden states causing the sensory input received by the network~\cite{friston2017active}. In particular, this can be modeled as an inference process which considers integrating prior expectations about the expected states of the world with real-time sensory data observations to infer the current state of the world~\cite{parr2022active}. In other words, the network updates its beliefs about the states of the world according to the sequence of sensory observations $o_{1:t}$ received from the real world.
The goal of this inference is to ensure the fit between the world model $ \mathcal{W}$ and the current state of the world at time $t$. The optimal posterior beliefs about these states are given by Bayes rule:
\begin{equation}\label{perception}
   p(s_{1:T}, \pi_o \mid o_{1:t}) = \dfrac{p( s_{1:T}, o_{1:t}, \pi_o)}{p(o_{1:t})}=\dfrac{p(o_{1:t} \mid s_{1:T}, \pi_o) p(s_{1:T}, \pi_o)}{p(o_{1:t})}.
\end{equation}
Here, the agent's policy $\pi_o$ is explicitly included in~\eqref{perception} not as a latent variable to be inferred, but as a fixed conditioning variable that specifies the actions taken under the world model, allowing the posterior to evaluate how well the world model's predictions, under $\pi_o$, align with the actual observations.


Upon perceiving the current state of the world at time $t$, the agent then exploits its \acp{DT} at the network edge to \emph{plan} the consequences of alternative action sequences $\pi \in \Pi$ through the world model. This planning entails the simulation of future trajectories which allow the physical \ac{AI} agent to counterfactually reason which actions best resolve prediction error in this unforeseen scenario.
For instance, in Fig.~\ref{Passive_to_Active_Inference}, we show an unforeseen jaywalking scenario in which a pedestrian crosses the street at a green light in front of the \ac{AI} agent (an autonomous vehicle). During this scenario, the prediction error increases because this unforeseen scenario was not observed during training. One way to reduce the prediction error in this scenario is by directing the \ac{AI} agent to avoid collision with the pedestrian.
In fact, this implies that the \ac{AI} agent must realize certain \emph{preferences} that minimize its expected surprise over its future trajectory $\tau > t$. Technically, this is a path-integral of a surprisal, which essentially maps to seeking the path of least action~\cite{friston2023path}. Preferred states of the world are those that the agent will likely tend to occupy, which combine two streams: (i) states that preserve the integrity of the \ac{AI} agent (e.g., protecting humans, avoid crashing, etc.) and, (ii) situations \emph{reinforced} by experience to become familiar under $\pi_o$  (i.e., preferences that the world admits or allows). Since $\mathcal{W}$ assigns higher likelihood to preferred states, steering the \ac{AI} agent towards these states naturally reduces the mismatch between predicted and actual observations, thereby minimizing surprise~\cite{schwartenbeck2015evidence}. Hence, it is essential to evaluate the surprise resulting from these alternative trajectories or paths at future time instants $\tau > t$.
Thus, the \emph{surprise} anticipated over each planned trajectory $\pi$ will be:
\begin{equation}\label{planning_eq}
   \mathscr{S}(\pi)=  \sum_{\tau=t+1}^{T} \mathscr{S} (\tau, \pi) = - \sum_{\tau = t+1}^{T}  \mathbb{E}_{p(o_{\tau} \mid \pi)} \ln \left[  \sum_{s_{\tau}} p(o_{\tau}, s_{\tau} \mid \pi)\right] = \sum_{\tau = t+1}^{T} H\left[ p(o_{\tau} \mid \pi)   \right],
\end{equation}
where $H[\cdot]$ is the entropy function, which captures the expected surprise.

This reasoning about expected surprise must then be leveraged by the \ac{AI} agent to return to its preferred states of the world. This restores states of the world whereby $\pi_o$ becomes once again the optimal policy\footnote{This offers a first-principles explanation for Kahneman's observation that System~1 naturally dominates System~2~\cite{kahneman2011thinking}.}.
To drive the \ac{AI} agent to take the actions that push it towards its preferred states, it is necessary to evaluate the level to which each action of the \ac{AI} agent can reduce surprise.
This can be captured through the \emph{surprise per action-value}, defined as:
\begin{equation}
Q_{\mathscr{S}}(a,s_t) =  \sum_{\pi \in \Pi} \delta(a, \pi(t)) \mathscr{S} (\pi) \quad \forall a \in \mathcal{A},  
\label{Q-surprise}
\end{equation}
where $\delta (\cdot)$ is the Kronecker delta function. The value in~\eqref{Q-surprise} represents the accumulated surprise (i.e., prediction error) resulting from taking action $a$ in state $s_t$. Thus,  $Q_{\mathscr{S}}(a,s_t)$ reflects the reward attributed to achieving the general survival objective.
For the \ac{AI} agent to update its policy $\pi_o$ with reasoning, the network shares a feedback $\boldsymbol{u}_t=[\gamma, Q_{\mathscr{S}} (a_{1},s_t), \ldots, Q_{\mathscr{S}} (a_{|\mathcal{A}|},s_t)]$ that transmits the parameters behind this reasoning to the \ac{AI} agent. Here, $\gamma\in [0, +\infty)$ is a policy precision parameter that reflects the neuro-modulatory mechanisms that determine the confidence of the network in the planning process and $|\mathcal{A}|$ is the cardinality of $\mathcal{A}$. Now, the \ac{AI} agent can interpret this feedback as additional evidence generated by reasoning, thinking, or even imagination about the world to update its beliefs about its base policy $\pi_o$. To see how this inference process scales the policy at test time, we next explain the biological process by which the brain updates the policy with such reasoning. This subsequent discussion will allow us to ground the belief update of the policy in a neuromimetic manner. 

\begin{figure}
	\centering
	\includegraphics[width=\columnwidth]{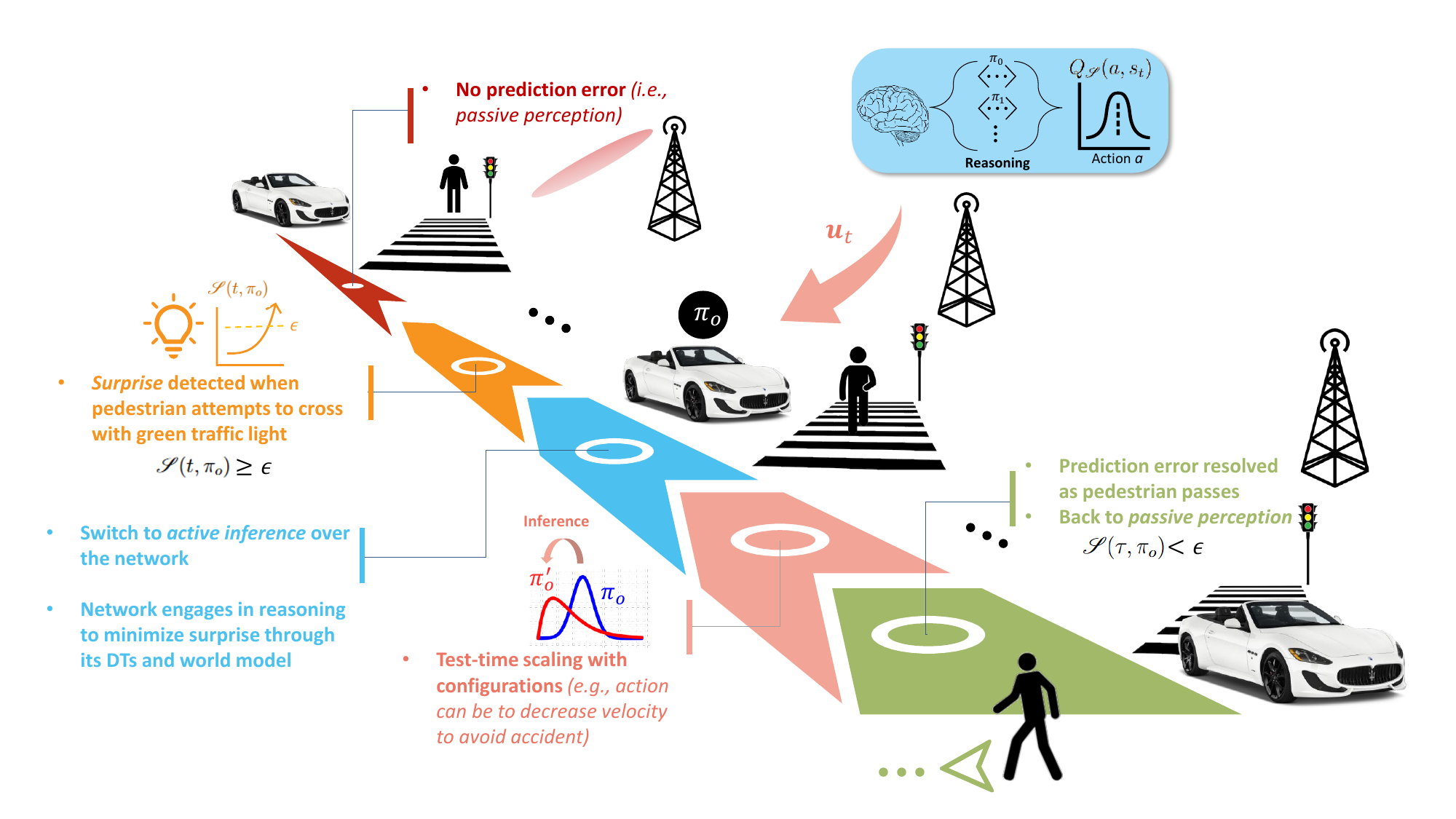}
	\caption{\small{Illustration of the test-time scaling mechanism in which the \ac{AI} agent must generalize in an unforeseen scenario that considers a jaywalking pedestrian appearing at test time.}}
	\label{Passive_to_Active_Inference}
	\vspace{-0.50cm}
\end{figure}

In training an \ac{AI} agent using techniques such as \ac{RL}, dopamine initially settles at its baseline level as a given policy converges, thereby allowing the agent to operate in feed-forward mode with its base policy. In an unforeseen scenario, dopamine spikes again due to prediction errors that violate predictions under the agent's learned world model. This spike signals shifts control from the \ac{BG} which governs policy execution to the \ac{PFC} that initiates deliberative reasoning. In particular, the \ac{PFC} responds by transitioning to active inference which includes planning (as inference) that generates counterfactual simulations of imagined action trajectories evaluated through alternative policies. The \ac{PFC} then transmits the resulting reasoning signal back to the \ac{BG} to modulate the value function of the policy in proportion to the prediction error.
Conceptually, this is similar to the reasoning feedback from the wireless network to the \ac{AI} agent.
Henceforth, this modulation captures the \emph{test-time scaling mechanism} for physical \ac{AI} agents which updates their policy $\pi_o (a|s_t)$ into $\pi_o' (a|s_t) \doteq p(a \mid o_{t+1}, s_t)$ that generalizes to the unforeseen scenario. Effectively, this test-time scaling translates into an update of beliefs about the policy with counterfactual evidence from reasoning which can be formulated as a Bayesian inference process\footnote{The term \emph{inference} is used in two distinct senses throughout this paper. The first is psychological, referring to the deliberative reasoning of System 2 as opposed to the habitual execution of System 1, following Kahneman's dual-process theory~\cite{kahneman2011thinking}. The second is mathematical, referring to Bayesian inference as the process of updating beliefs to minimize prediction error, which is the computational mechanism underlying active inference.}.
Unlike perception that admits inference based on sensory observations $o_{1:t}$, policy scaling relies on counterfactual observations $o_{t+1:T}$; namely, evidence generated from the world model $\mathcal{W}$ through imagined action trajectories $\pi$. 
An inference with such counterfactual evidence yields a soft Bayesian belief update through \emph{Pearl’s method of virtual evidence}~\cite{pearl2014probabilistic} (also known as Jeffrey's rule~\cite{jacobs2019mathematics}). This, in turn, gives rise to posterior scaled policy of the \emph{test-time scaling law} that is established next.
\vspace{-0.2cm}

\begin{theorem}[Test-Time Scaling Law]
The \emph{test-time scaling mechanism} that scales policy $\pi_o (a|s_t)$ to $\pi_o' (a|s_t)$ with inference about the actions that are likely to minimize prediction error (i.e., surprise)  is given by:\vspace{-0.2cm}
\begin{align}
    \underbrace{\pi_{o}'(a \mid s_t)}_{\mathrm{posterior}} &\propto 
    \underbrace{\pi_o(a \mid s_t)}_{\mathrm{prior}}  \underbrace{p(o_{t+1:T} \mid \pi)}_{\mathrm{likelihood}} \nonumber\\
    &\quad = \underbrace{[\pi_o (a \mid s_t)]^{1-\theta}  \big[\pi_o (a \mid s_t) \exp \big(-\gamma Q_{\mathscr{S}}(a, s_t) \big)\big]^{\theta}}_{\mathrm{soft\, Bayesian\, update}}          \nonumber \\
    & \quad = \underbrace{\pi_o(a \mid s_t)}_{\mathrm{feed-forward}}  \underbrace{\exp \big(-\gamma \theta Q_{\mathscr{S}}(a, s_t) \big)}_{\mathrm{inference}},
    \label{scaling}
\end{align}
\end{theorem}
\noindent where $\theta \in [0,1]$ is the normalized surprise that bounds $ \mathscr{S} (t, \pi_o)$ in the unit interval (this parameter will be mathematically defined in Section~\ref{planning}) to control the scaling of the base policy.

The novel, fundamental insight that we can draw from the test-time scaling in~\eqref{scaling} stems from probabilistic inference. In particular, the $\exp\left(-\gamma  Q_{\mathscr{S}}(a, s_t)\right)$ term transforms the value function of the actions into likelihood probabilities, analogous to the Boltzmann distribution. 
Since planning aims to minimize the anticipated surprise $\mathscr{S}(\pi)$, which accumulates as a path integral of entropies over future trajectories as defined in~\eqref{planning_eq}, the likelihood of future observations under a planned trajectory is naturally modeled as an exponential distribution\footnote{This choice is the maximum entropy distribution subject to a constraint on the expected sum of entropies, and therefore makes the fewest assumptions beyond what the planning objective already prescribes.}. 
Clearly, when the world is mostly predictable, the surprise  $\mathscr{S}(t, \pi_o) \approx 0 \implies \theta \approx 0$. In this case, the inference term in~\eqref{scaling} vanishes. Consequently, this reaffirms that action selection reduces into feed-forward control without deliberative thinking (i.e., reasoning). Importantly, \emph{the policies in today's standard \ac{RL} formulations remain limited to this case where expected surprise is ignored due to stationary assumptions~\cite{sutton2018reinforcement, 
mnih2015human, chandak2020optimizing}, which is clearly a special case of the equation in \eqref{scaling}.}

\begin{remark}
It is useful to clarify that the final form of~\eqref{scaling} resembles Boltzmann policies commonly encountered in \ac{RL} where action probabilities are proportional to action-value $\exp(Q(s,a))$~\cite{sutton2018reinforcement}, as well as to energy-based models of the form $p(x) \propto \exp(-E(x))$~\cite{lecun2006tutorial}. However, the derivation and interpretation here are fundamentally distinct. First, the update arises naturally as a soft Bayesian posterior update through Pearl's virtual evidence method, where the exponential scaling is the virtual evidence likelihood rather than a hand-crafted energy function. Second, the update is multiplicative over the prior policy $\pi_o(a|s)$, reflecting belief refinement rather than a from-scratch softmax over actions. Third, $Q_{\mathscr{S}}(a, s_t)$ is not a static action-value estimate but a surprise per action value function that aggregates the expected surprise cost of taking action $a$ in state $s_t$ at test time. This makes the effective precision (i.e., inverse temperature) $\gamma\theta$ a dynamic, surprise-modulated quantity rather than a fixed hyperparameter. The similarity to Boltzmann formulations is therefore a consequence of the underlying Bayesian geometry, and not an assumption imported from the energy-based modeling literature. 
\end{remark}

Indeed, one could argue that the absence of the test-time scaling mechanism in~\eqref{scaling} has prevented today's physical \ac{AI} agents from generalizing in unforeseen scenarios. 
Evidently, initiating this scaling operation first requires performing the necessary perception and planning processes. Nevertheless, this can be challenging as calculating the model evidence $p(o_{1:t})$ in \eqref{perception} for perception can be computationally intractable in practice. This is due to the fact that calculating this distribution requires marginalizing over all the possible states, which can be infeasible given the large number of states in complex models of the world. Similarly, calculating the future surprise while planning includes summations over the set of states and observations that can be analytically intractable. This, in turn, limits the scaling of the policies and their generalization at test time.

To address this challenge, one can call on variational Bayes that furnishes tractable solutions for perception and planning. This introduces \emph{free energy} as an upper bound on surprise, which can be directly optimized to converge on the posterior while eliding marginalization. Here, a key observation is that both perception and planning share the same underlying goal. On the one hand, perception aims to infer the best explanation behind the sensory observations. Equivalently, perception in~\eqref{perception} amounts to maximizing model evidence for $\tau=t$. On the other hand, planning aims to find the action sequence that minimizes the future surprise for $\tau > t$. Clearly, minimizing surprise is equivalent to maximizing model evidence under the preferred world model $\mathcal{W}$. Thus, perception and planning coincide in the context of model evidence maximization. As such, combining both operations under this single objective allows formulating the inference problem in terms of free energy minimization that essentially serves as a bound on (log) model evidence or surprisal. Intuitively, action and learning should further follow this free energy formulation as they can be modeled in terms of inference.

Next, we will describe the (variational) active inference solution summarized in Fig.~\ref{Solution_Roadmap} that unifies perception, planning, action selection, and learning under a single objective of \emph{free energy minimization}.

\begin{figure}
	\centering
	\includegraphics[width=\columnwidth]{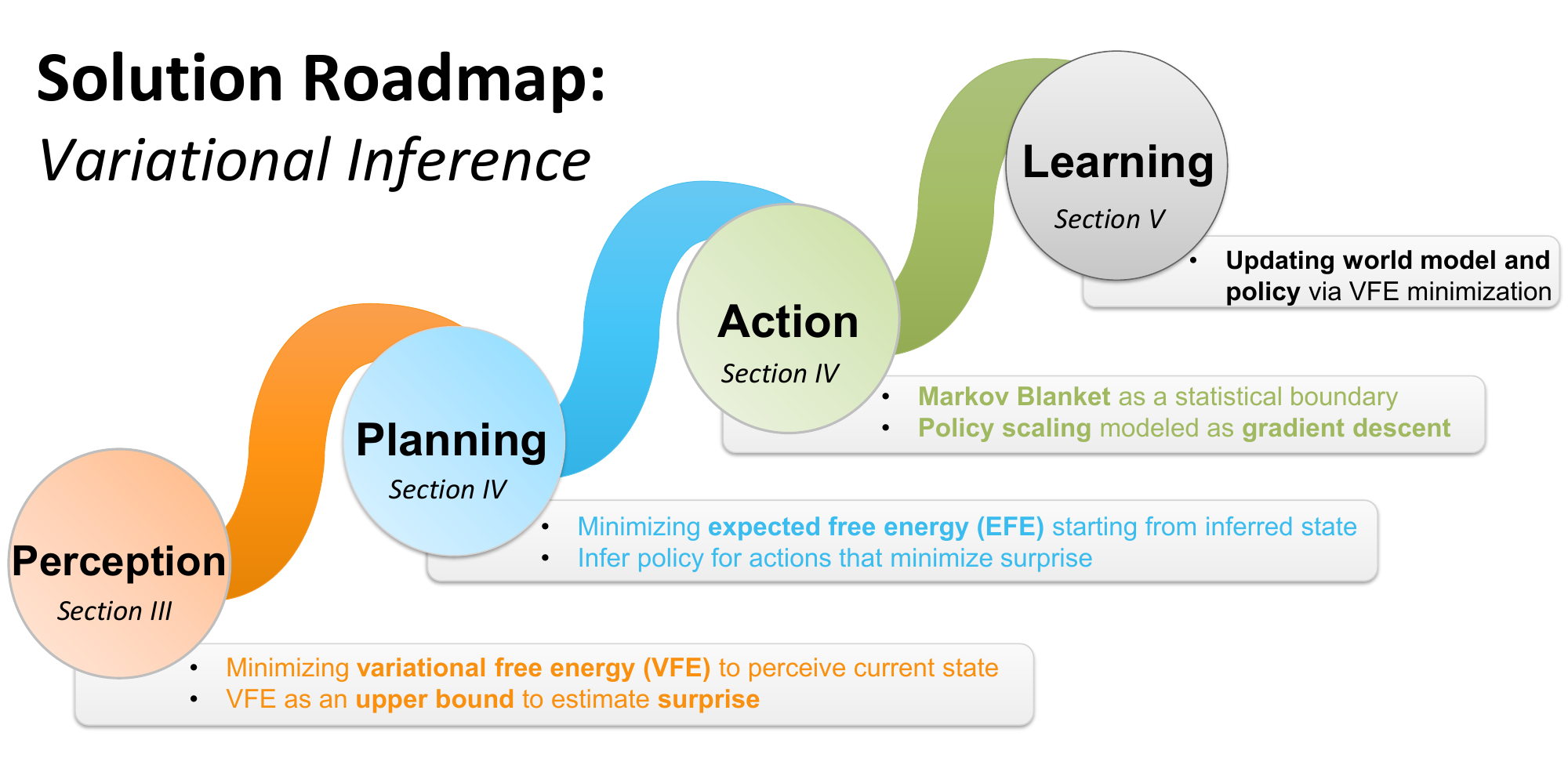}
	\caption{\small{The roadmap for the proposed variational inference solution that unifies perception, planning, action, and learning to enable the test-time scaling framework for the physical \ac{AI} agent.}}
	\label{Solution_Roadmap}
	\vspace{-0.50cm}
\end{figure}

\section{Perception as Inference: Minimizing Variational Free Energy}
\label{Perception}



In this section, we adopt a variational Bayesian inference framework~\cite{jordan1999introduction} to obtain a tractable solution for perception and surprise estimation. This requires introducing a variational distribution to approximate the posterior over states in~\eqref{perception}. In particular, an approximation to the true posterior is obtained by minimizing the \ac{VFE} via variational inference, which simultaneously provides an upper bound on surprise. Thus, by building on this formulation, we cast this inference problem as a \ac{POMDP}. To solve this \ac{POMDP} and render inference computationally feasible, we then employ a mean-field approximation\cite{wainwright2008graphical}. This allows adopting \ac{VMP}~\cite{winn2005variational} as an efficient graphical and algorithmic scheme for solving the resulting \ac{POMDP} and determining the corresponding surprise.

First, we introduce a variational distribution $q(s_{1:T}, \pi_o)$ to approximate the true posterior $p(s_{1:T}, \pi_o \!\mid \!o_{1:t})$ in~\eqref{perception}. This approximation is captured by the \ac{KL} divergence between the variational distribution and the true posterior: $D_{\mathrm{KL}}\big[ q(s_{1:T},  \pi_o) \!\,\|\, p(s_{1:T}, \pi_o \mid o_{1:t}) \big]$. This divergence is then minimized through the \ac{VFE} functional $F[q(s_{1:T}, \pi_o)]$ (also known as the negative evidence lower bound (ELBO) in variational inference~\cite{kingma2014auto}), defined next.
\begin{definition}
The \emph{\ac{VFE}} is defined as the divergence between the variational distribution $q(s_{1:T}, \pi_o)$ and generative world model $\mathcal{W}(s_{1:T}, o_{1:t}, \pi_o)$:
\begin{align}
F[q(s_{1:T}, \pi_o)] &= D_{\mathrm{KL}}\big[ q(s_{1:T}, \pi_o) \,\|\, p(s_{1:T}, o_{1:t}, \pi_o) \big] \nonumber \\ 
&= \mathbb{E}_{q(s_{1:T}, \pi_o)}\!\left[ \ln q(s_{1:T}, \pi_o) - \ln p(s_{1:T}, o_{1:t}, \pi_o) \right] \nonumber \\ 
&= D_{\mathrm{KL}}\big[ q(s_{1:T}, \pi_o) \,\|\, p(s_{1:T}, \pi_o) \big] - \mathbb{E}_{q(s_{1:T}, \pi_o)}\!\left[ \ln p(o_{1:t} \mid s_{1:T}, \pi_o) \right] \nonumber.
\end{align}
\label{VFE} 
\end{definition} 
\vspace{-0.7cm}
\noindent Based on Definition~\ref{VFE}, we can show that the \ac{VFE} admits a tractable solution to approximate surprise.
\begin{lemma}
The \ac{VFE} provides an upper bound on surprise (or lower bound on model evidence), i.e., $\mathscr{S} (t, \pi_o) \leq F[q(s_{1:T}, \pi_o)].$
    \begin{IEEEproof}
        See Appendix A.
    \end{IEEEproof}
    \label{surpise}
\end{lemma} 

From Definition~\ref{VFE} and Lemma~\ref{surpise}, we observe that by varying $q(s_{1:T}, \pi_o)$ to minimize the \ac{VFE}, it is possible to approximate $p(s_{1:T}, \pi_o \!\mid \!o_{1:t})$, while simultaneously approaching an accurate estimate for $\mathscr{S} (t, \pi_o)$. This means that the network initially perceives the physical world as it captures surprise. 
If surprise exceeds the threshold $\epsilon$, the network engages in deliberative reasoning, whereby it plans alternative policies $\pi \in \Pi$ that can minimize the surprise over future time instances to scale the policy $\pi_o$ of the \ac{AI} agent. Otherwise, the network moves forward by just considering the policy $\pi_o$.
This sequential inference process can be naturally formulated as a \ac{POMDP}~\cite{smith2022step}. This is due to the fact that the network can only perceive the current state of its world (i.e., \ac{AI} agent and assets) through sensory observations and must infer the causes behind these observations to update its beliefs about the state of the world. Accordingly, the network must act under uncertainty about the hidden states that generate sensory observations. In contrast to \ac{RL}, which exclusively aims to maximize its long-term task reward, this active inference framework subsumes the task reward while ensuring survival by minimizing the \ac{VFE} as an upper bound on prediction errors.
We thus define a \ac{POMDP} to model this problem, with the following key components:

\begin{itemize}
    \item \emph{Agent:} The network is the agent that performs inference and reasoning on behalf of the physical \ac{AI} agent in the real world. In other words, the \ac{AI} agent offloads its reasoning process onto the network edge.
    \item \emph{Observations:} At each time step $\tau$, the network receives sensory observations $o_\tau \in \mathcal{O}$ from the real world.
    \item \emph{States:} The network infers the state of the real world $s_\tau \in \mathcal{S}$ at time instant $\tau$ from its sensory observations  $o_\tau$ through perception.
    \item \emph{Observation model}: The likelihood function $\boldsymbol{A}$ that maps the sensory observations $o_\tau \in \mathcal{O}$ into the states $s_\tau \in \mathcal{S}$, with its entries parametrized as $p(o_\tau \mid s_\tau).$

    \item \emph{Dynamics (transition) model:} The evolution of hidden states is governed by the transition model $\boldsymbol{B}_{\pi, \tau}$ with its entries parameterized as $p(s_{\tau} \mid s_{\tau-1}, \pi) $ $\forall s_\tau \in \mathcal{S}$, where the actions $a_{\tau} \in \mathcal{A}$ performed by the network at time $\tau$ come from the policies defined as the sequences of future actions $\pi \in \Pi$. 
    
    \item \emph{Actions:} The network shares a feedback configuration $\boldsymbol{u}_\tau \in \mathcal{U}$ to scale the policy of the \ac{AI} agent in case of surprise. Otherwise, no feedback is shared in the absence of surprise.
\end{itemize}


\begin{figure}
	\centering
	\includegraphics[width=0.85\columnwidth]{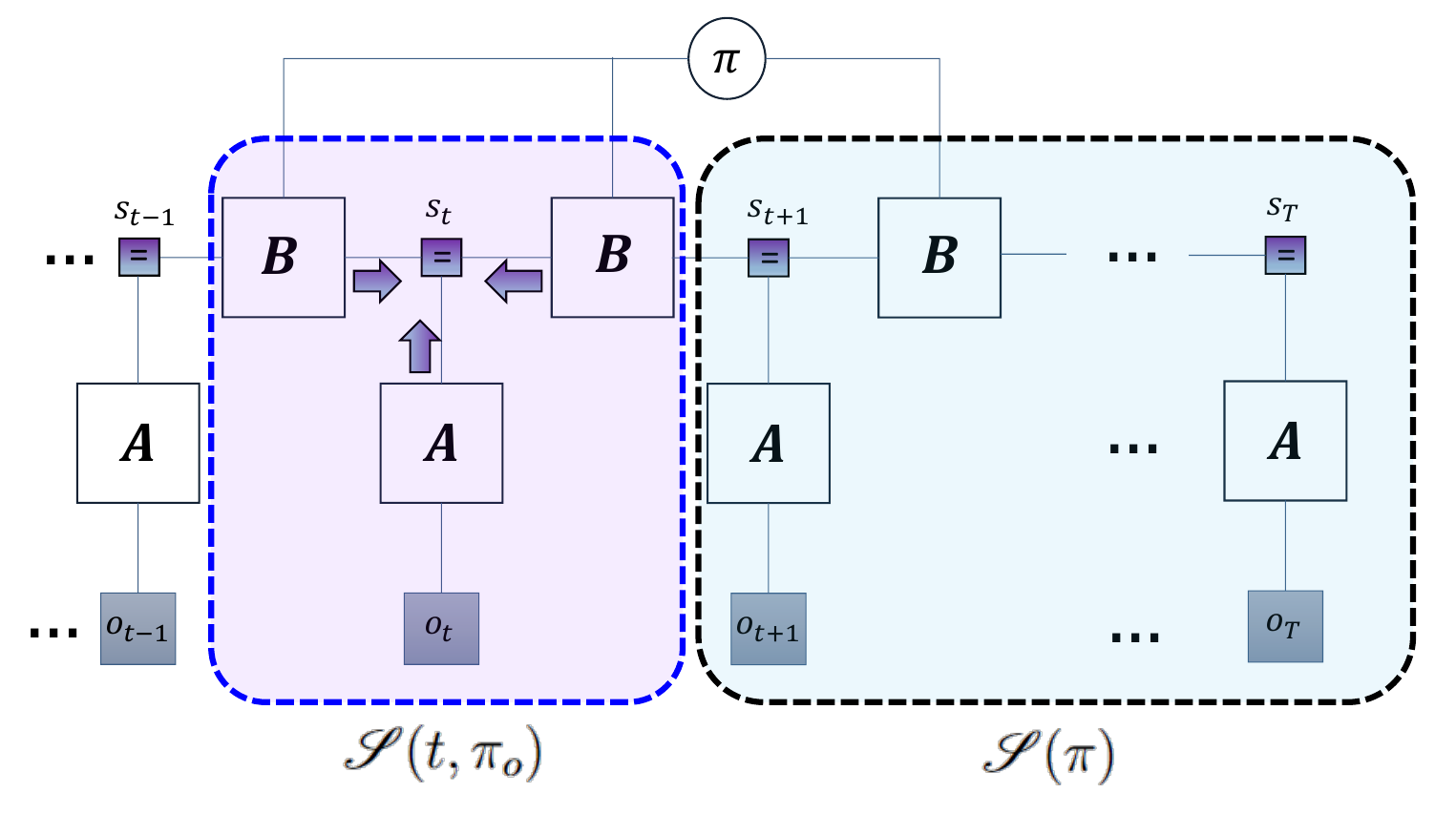}
	\caption{\small{Forney-style factor graph representing the generative model of the \ac{POMDP}.}}
	\label{Factor_Graph}
	\vspace{-0.50cm}
\end{figure}

To solve this \ac{POMDP} with approximate inference, we adopt a \ac{VMP} scheme that can efficiently approximate the true posteriors by minimizing the \ac{VFE}. This is illustrated in Fig.~\ref{Factor_Graph} by leveraging a Forney-style factor graph representation~\cite{forney2001codes} of the generative world model $\mathcal{W}(s_{1:T}, o_{1:t}, \pi)$ defined in \eqref{world model}. In particular, approximate Bayesian inference is implemented through the propagation of messages between neighboring nodes, that correspond to the local updates of their sufficient statistics. Effectively, \ac{VMP} provides efficient, distributed, and biologically inspired inference updates, as elucidated next.

To derive the variational inference updates, we minimize the \ac{VFE} in \eqref{VFE} with respect to $q(s_{1:T}, \pi_o)$. To ensure tractability of this approximation, the variational distribution $q(s_{1:T}, \pi_o)$ is typically assumed to factorize under a \emph{mean-field approximation} into the product of independent distributions, such that: \vspace{-0.5cm}
\begin{equation}
q(s_{1:T}, \pi_o) \approx q(\pi_o)\prod_{\tau = 1 }^{T} q(s_{\tau} | \, \pi_o).
\label{MFA_0}
\end{equation}
The approximation in~\eqref{MFA_0} assumes that states of the world conditioned on the base policy $q(s_{\tau} | \, \pi_o)$ $\forall \tau \in \{1, \ldots, T\}$ are independent across time.
As perception is concerned with the estimation of states $q(s_{1:T})$, we can consider $q(\pi_o)$ as a constant term and safely neglect it during this belief update. Thus, we can rewrite the \ac{VFE} in Definition~\ref{VFE} to isolate $q(s_{\tau}\mid \pi_o)$ as follows:
\begin{align}
    F[q(s_{1:T}, \pi_o)] &\propto \mathbb{E}_{q(s_{1:T},  \pi_o)}\!\left[ \ln q(s_{\tau} | \,\pi_o) \right] + \mathbb{E}_{q(s_{1:T},  \pi_o)}\!\left[ \ln \prod_{\kappa \neq \tau} q(s_{\kappa} | \,\pi_o) \right] -\mathbb{E}_{q(s_{1:T},  \pi_o)}\!\left[ \ln p(s_{1:T},  o_{1:t}, \pi_o)  \right] \nonumber\\
    &= \mathbb{E}_{q(s_{\tau} |  \pi_o)}\!\left[ \ln q(s_{\tau} | \,\pi_o) \right] + \mathbb{E}_{q(s_{-\tau} | \, \pi_o)}\!\left[ \ln \prod_{\kappa \neq \tau} q(s_{\kappa} | \,\pi_o) \right] -\mathbb{E}_{q(s_{1:T},  \pi_o)}\!\left[ \ln p(s_{1:T},  o_{1:t}, \pi_o)  \right].
    \label{MFA}
    \vspace{-0.45cm}
\end{align}
In~\eqref{MFA}, $\mathbb{E}_{q(s_{-\tau} | \, \pi_o)} [\cdot]$ is the expectation over all factors excluding $q(s_{\tau} | \, \pi_o)$.
Here, $q(s_{\tau} | \,\pi_o)$ is constant for all variables in $q(s_{1:T},  \pi_o)$ except $q(s_{\tau} | \,\pi_o)$ and $\prod_{\kappa \neq \tau} q(s_{\kappa} | \,\pi_o)$ is constant with respect to $q(s_{\tau} | \,\pi_o)$. 
To minimize $ F[q(s_{1:T}, \pi_o)]$ with respect to $q(s_{\tau} | \, \pi_o)$, we rearrange ~\eqref{MFA} to reach the Euler–Lagrange step that gives the coordinate ascent variational update:
\begin{align}
    F[q(s_{1:T}, \pi_o)] &= \mathbb{E}_{q(s_{\tau} |  \pi_o)}\!\left[ \ln q(s_{\tau} | \,\pi_o) \right] -\mathbb{E}_{q(s_{1:T},  \pi_o)}\!\left[ \ln p(s_{1:T},  o_{1:t}, \pi_o)  \right] + K \nonumber \\
    &= \mathbb{E}_{q(s_{\tau} |  \pi_o)}\! \left[ \ln q(s_{\tau} | \,\pi_o) - \mathbb{E}_{q(s_{-\tau} | \, \pi_o)} [ \ln p(s_{1:T},  o_{1:t}, \pi_o)]  \right] + K \nonumber\\
    &= \mathbb{E}_{q(s_{\tau} |  \pi_o)}\! \left[ \ln q(s_{\tau} | \,\pi_o) -  \ln q^*(s_{\tau} | \,\pi_o)  \right] + K \nonumber\\
    &= D_{\mathrm{KL}}\big[ q(s_{\tau} |  \pi_o)   \,\|\,  q^*(s_{\tau} | \,\pi_o) \big] + K.
    \vspace{-0.2cm}
    \label{Euler-Lagrange}
\end{align}
Thus, minimizing $F[q(s_{1:T}, \pi_o)]$ is equivalent to minimizing $D_{\mathrm{KL}}\big[ q(s_{\tau} |  \pi_o)   \,\|\,  q^*(s_{\tau} | \,\pi_o) \big]$.
Here, $\ln q^*(s_{\tau} | \,\pi_o) \triangleq \mathbb{E}_{q(s_{-\tau} | \, \pi_o)} [ \ln p(s_{1:T},  o_{1:t}, \pi_o)]$ is the optimal posterior that minimizes  $F[q(s_{1:T}, \pi_o)] $ and $K$ represents the constant terms. Considering that only the factors in~\eqref{world model} that incorporate $s_{\tau}$ remain as a result of the expectation, then we can write the message passing equation: 
\begin{align}
\ln q^{*}(s_\tau \mid \pi_o)
&\propto \underbrace{\mathbb{E}_{q(s_{\tau-1}\mid\pi_o)}\!\left[\ln p(s_\tau \mid s_{\tau-1}, \pi_o)\right]}_{\text{prior forward transition}}
 + \underbrace{\mathbb{E}_{q(s_{\tau+1}\mid\pi_o)}\!\left[\ln p(s_{\tau+1} \mid s_\tau, \pi_o)\right]}_{\text{prior backward transition}}
+ \underbrace{\mathbf{1}_{\{\tau \leq t\}} \,\ln p(o_\tau \mid s_\tau)}_{\text{observation}} \nonumber\\
& \quad = \underbrace{\ln \boldsymbol{B}_{\pi_o,\tau-1}\, \boldsymbol{s}_{\pi_o,\tau-1}}_{\text{message 1}}
\;+\; \underbrace{\ln \boldsymbol{B}_{\pi_o,\tau}^{\top}\, \boldsymbol{s}_{\pi_o,\tau+1}}_{\text{message 2}}
\;+\; \underbrace{\ln \boldsymbol{A}^{\top}\, \boldsymbol{o}_{\tau}}_{\text{message 3}},
\label{message passing}
\end{align}
where $\boldsymbol{s}_{\pi_o, \tau-1}$, $\boldsymbol{s}_{\pi_o, \tau+1}$, and $\boldsymbol{o}_{\tau}$ represent the vector representations of the belief distributions $q(s_{\tau-1} \mid \pi_o)$, $q(s_{\tau+1} \mid \pi_o)$, and $q(o_{\tau} \mid \pi_o)$, respectively.
In~\eqref{message passing}, message~1 propagates beliefs forward through the transition model $\ln \boldsymbol{B}_{\pi_o, \tau-1} \boldsymbol{s}_{\pi_o, \tau-1}$, 
capturing what the agent expects the current state to be, based on its history. Also, message~2 propagates beliefs backward from future states through $\ln \boldsymbol{B}_{\pi_o, \tau}^{\top} \boldsymbol{s}_{\pi_o, \tau+1}$, capturing what the current state must have been given where the agent is heading. Message~3 grounds both predictions in actual sensory evidence through $\ln \boldsymbol{A}^{\top} \boldsymbol{o}_{\tau}$, incorporating the likelihood of the current observation under the observation model. Thus, the posterior belief $q^*(s_{\tau} \mid \pi_o)$ will be the sum of these three messages, combining past predictions, future expectations, and present observations into a unified state estimate.

To ensure that the expectations required for message updates in~\eqref{message passing} admit tractable and closed-form solutions, we consider that the world model belongs to the \emph{conjugate-exponential family}. 
This is reasonable since the conjugate-exponential family is the class of distributions where the posterior retains the same functional form as the prior after incorporating new observations~\cite{winn2005variational}. This property ensures that belief updates remain closed-form and tractable at every time step, which is a necessary condition for the message passing updates in~\eqref{message passing} to be computationally feasible as the agent accumulates observations over time.
In this formulation, both the transition and likelihood models in~\eqref{message passing} are \emph{categorical distributions}, while their parameters admit \emph{Dirichlet priors}~\cite{bishop2006pattern}. Herein, the categorical--Dirichlet conjugacy ensures that the resulting posteriors of the inferred states maintain the same functional form as their priors.
Effectively, each column of the likelihood matrix $\boldsymbol{A}$ and transition matrices $\boldsymbol{B}_{\pi, \tau} \forall \pi \in \Pi$ represents a categorical distribution, whose parameters are governed by Dirichlet distributions that evolve over time. This conjugacy allows natural parameters of exponential-family distributions to be passed as messages, enabling efficient inference and learning within the world model $\mathcal{W}$ and policy $\pi_o$. Accordingly, we will show in Section~\ref{Learning via Inference} how this conjugacy translates to enabling efficient and tractable parameter updates of the world model $\mathcal{W}$ and policy $\pi_o$. Notably, inference of states at time $\tau = 1$ and $\tau=T$ will drop messages $1$ and $3$ in \eqref{message passing}, respectively.

To transform the the log posterior in~\eqref{message passing} into a probability distribution, we perform a softmax to normalize the product of all passed messages, as follows:
\begin{equation}
\boldsymbol{s}_{\pi_o,\tau}
\;=\;
\sigma\!\Big( \ln \boldsymbol{B}_{\pi_o,\tau-1}\, \boldsymbol{s}_{\pi_o,\tau-1}
\;+\; \ln \boldsymbol{B}_{\pi_o,\tau}^{\top}\, \boldsymbol{s}_{\pi_o,\tau+1}
\;+\; \ln \boldsymbol{A}^{\top}\, \boldsymbol{o}_{\tau} \Big).
\end{equation}

While \ac{VMP} updates provide a closed-form solution for obtaining $q^{*}(s_\tau \mid \pi_o)$ under conjugacy, we note that we can alternatively express the same updates as a gradient descent on prediction errors. This highlights how perception as inference could be generally implemented even in the absence of conjugacy. Moreover, it is imperative to highlight this alternative solution to draw contrast with the gradient descent solution that minimizes the expected prediction errors through action, as will be explained further in Section~\ref{planning}. To implement this belief update dynamically, we introduce the state prediction error $\boldsymbol{\epsilon}_{\pi, \tau} = \ln \boldsymbol{B}_{\pi_o,\tau-1}\, \boldsymbol{s}_{\pi_o,\tau-1}
\;+\; \ln \boldsymbol{B}_{\pi_o,\tau}^{\top}\, \boldsymbol{s}_{\pi_o,\tau+1}
\;+\; \ln \boldsymbol{A}^{\top}\, \boldsymbol{o}_{\tau}  -  \ln \boldsymbol{s}_{\pi_{o}, \tau}$. This prediction error corresponds to the \ac{VFE}'s rate of change w.r.t. $\boldsymbol{s}_{\pi_{o}, \tau}$, i.e., $\boldsymbol{\epsilon}_{\pi_{o}, \tau} = - \dfrac{\partial \boldsymbol{F}_{\pi_o}}{\partial \boldsymbol{s}_{ \pi_{o}, \tau}}$. This can be seen from the rearrangement of the \ac{VFE} by substituting \eqref{message passing} into~\eqref{Euler-Lagrange}, as follows:
\begin{align}
F[q(s_{1:T}, \pi_o)] 
&\propto \mathbb{E}_{q(s_{\tau} |  \pi_o)} \Big[ \ln q(s_{\tau} | \,\pi_o) 
- \mathbb{E}_{q(s_{\tau-1}\mid\pi_o)}\!\left[\ln p(s_\tau \mid s_{\tau-1}, \pi_o)\right] \nonumber \\
 & \qquad \qquad + \mathbb{E}_{q(s_{\tau+1}\mid\pi_o)}\!\left[\ln p(s_{\tau+1} \mid s_\tau, \pi_o)\right]
+ \mathbf{1}_{\{\tau \leq t\}} \,\ln p(o_\tau \mid s_\tau) \Big] \nonumber \\
 &= \boldsymbol{s}_{\pi_{o}, \tau}. \Big[ \ln \boldsymbol{s}_{\pi_{o}, \tau} -\ln \boldsymbol{B}_{\pi_o,\tau-1}\, \boldsymbol{s}_{\pi_o,\tau-1}
\;-\; \ln \boldsymbol{B}_{\pi_o,\tau}^{\top}\, \boldsymbol{s}_{\pi_o,\tau+1}
\;-\; \ln \boldsymbol{A}^{\top}\, \boldsymbol{o}_{\tau} \Big] \nonumber \\ 
& =  - \boldsymbol{s}_{\pi_{o}, \tau} \;. \; \boldsymbol{\epsilon}_{\pi_{o}, \tau}.
\label{gradient}
\end{align}
Hence, $\boldsymbol{\epsilon}_{\pi_{o}, \tau} = - \dfrac{F[q(s_{1:T}, \pi_o)]}{\boldsymbol{s}_{\pi_{o}, \tau}} \approx - \nabla_{\boldsymbol{s}_{ \pi_{o}, \tau}}  \boldsymbol{F}_{\pi_o}$ when considering that each message passing update conforms to a gradient step that transitions $\boldsymbol{s}_{\pi_o,\tau}$ towards the fixed-point solution in~\eqref{message passing}. 
Thus, this solution is iteratively reached by iterating the following gradient steps in dual (logit) coordinates until $\boldsymbol{\epsilon}_{\pi_{o}, \tau}$ is minimized: 
\begin{align}
 \left\{
\begin{aligned}
    & \boldsymbol{\epsilon}_{\pi_{o}, \tau} \leftarrow \ln \boldsymbol{B}_{\pi_o,\tau-1}\, \boldsymbol{s}_{\pi_o,\tau-1}
    \;+\; \ln \boldsymbol{B}_{\pi_o,\tau}^{\top}\, \boldsymbol{s}_{\pi_o,\tau+1}
    \;+\; \ln \boldsymbol{A}^{\top}\, \boldsymbol{o}_{\tau} \; - \ln \boldsymbol{s}_{\pi_{o}, \tau}, \\
    & \ln \boldsymbol{s}_{\pi_{o}, \tau} \leftarrow \ln \boldsymbol{s}_{\pi_{o}, \tau} + \boldsymbol{\epsilon}_{\pi_{o}, \tau}, \\ 
    &\boldsymbol{s}_{\pi_{o}, \tau} \leftarrow \sigma (\ln \boldsymbol{s}_{\pi_{o}, \tau}).
\end{aligned}
 \right . 
 \label{iterative}
 \end{align}
 
This operation is performed over all edges $\boldsymbol{s}_{\pi_o, \tau} \;\forall \tau \in \{1, \ldots, T\}$ of the factor graph in Fig.~\ref{Factor_Graph} until the difference between
updates converges to an acceptably low value. Accordingly, when $\boldsymbol{\epsilon}_{\pi_{o}, \tau}$ is minimized, it is possible to approximate the posterior as $q^{*}(s_{1:T}, \pi_o)\approx p(s_{1:T}, \pi_o \!\mid \!o_{1:t})$ according to \eqref{MFA_0}. In this case, the surprise can be estimated as
$\mathscr{S} (t, \pi_o) \simeq F[q(s_{1:T}, \pi_o)] =  \sum_{\tau} - \boldsymbol{s}_{\pi_{o}, \tau} \;. \;  \boldsymbol{\epsilon}_{\pi_{o}, \tau}$ by considering the states $\boldsymbol{s}_{\pi_{o}, \tau}$ over all time instants $\tau \in \{1, \ldots, T\}$. 
Thus, the message passing updates in~\eqref{message passing} and the gradient descent on prediction errors in~\eqref{iterative} are two equivalent perspectives on the same inference process. In particular, the former operates as belief propagation over a graphical model such as Fig.~\ref{Factor_Graph}, while the latter corresponds to (a form of) predictive coding in neural systems that performs a local gradient descent to minimize the mismatch between predicted and observed states.

After the network estimates the posterior belief over states $q^{*}(s_{1:T}, \pi_o)$ and the surprise $\mathscr{S} (t, \pi_o)$, the network can now leverage these estimates to engage in planning and action selection. In the following section, we will formulate planning and action as free energy minimization that aims to minimize prediction error over future trajectories. By the end of this process, the physical \ac{AI} agent scales their policy $\pi_o$ to generalize in the unforeseen scenario.

\section{Planning \& Action as Inference: Minimizing Expected Free Energy}
\label{planning}


When the prediction error exceeds its threshold $\epsilon$ during perception, the \ac{AI} agent can infer that it is facing an unforeseen scenario. The network responds by engaging in counterfactual reasoning over alternative policies in an attempt to minimize this prediction error over future time instants. As both perception and planning aim to minimize prediction error (or equivalently maximize model evidence), it is possible to view planning as an extension of perception. While \ac{VFE} quantifies the surprise incurred by the \ac{AI} agent at the current time $t$ given past observations, the \ac{EFE} is introduced to extend this notion forward in time by computing the surprise that the agent anticipates incurring under each policy $\pi \in \Pi$ over future trajectories. Thereby, \ac{EFE} shifts inference from just explaining past and current observations $o_{1:t}$ to optimizing future behavior $\tau > t$. Thus, the network evaluates these policies $\pi$ according to the \ac{EFE} to further capture the surprise per action values $Q_{\mathscr{S}}(a,s_t) ,\forall a\in \mathcal{A}$. Subsequently, the network shares these values back with the physical \ac{AI} agent as an inferred policy $\psi$ to scale its policy $\pi_o$ by which it can take action. 


To determine the \ac{EFE} that bounds surprise over future time instants\footnote{Since free energy upper bounds surprise and the Boltzmann distribution is the canonical distribution over energy-bounded quantities, this provides additional theoretical evidence that modeling the likelihood of $Q_{\mathscr{S}}(a, s_t)$ as a Boltzmann distribution in~\eqref{scaling} is well-founded.}, we extend the notion of \ac{VFE} in~\eqref{VFE} to future time steps $\tau>t$. As future observations $o_{t+1:T}$ have not been encountered yet, they can be treated as a random variable~\cite{millidge2021whence}. Accordingly, we can define the \emph{per-step \ac{EFE}} at time $\tau$ as follows:
\begin{equation}
G (\tau, \pi_o)
:= \mathbb{E}_{q(o_\tau,s_\tau\mid\pi_o)}\!\big[\ln q(s_\tau\mid\pi_o) - \ln p(o_\tau,s_\tau\mid\pi_o)\big].
\label{EFE}
\end{equation}
Then, we can find the \ac{EFE} under a fixed policy $\pi_o$ over the future horizon $t < \tau \leq T$ according to the following proposition.
\begin{proposition}
The \ac{EFE} can be decomposed into a measure of divergence (i.e., risk of deviation) between 
the expected observations from $\pi_o$ and the preferences of the \ac{AI} agent 
that are encoded as a likelihood over observations (i.e., $p(o_{\tau}\mid C))$ plus the 
expected ambiguity (i.e., likelihood entropy), as follows:
    \begin{equation}
     G (\tau, \pi_o) 
    = \underbrace{D_{\mathrm{KL}}\big[ q(o_{\tau} \mid  \pi_o) \,\|\, p(o_{\tau} \mid C) 
    \big]}_{\text{goal-directed behavior (preference matching)}} + 
    \underbrace{\mathbb{E}_{q(s_\tau\mid\pi_o)} H\Big[p(o_\tau\mid 
    s_\tau)\Big]}_{\text{epistemic value (uncertainty reduction)}},
    \end{equation} 
    where the model evidence depends on the preferences $C$ encoded within $\mathcal{W}$ to shape the prior preference \(p(o_\tau\mid C)\).
    \begin{IEEEproof}
        See Appendix B.
    \end{IEEEproof}
    \label{EFE}
\end{proposition}
Proposition~\ref{EFE} shows that planning is simultaneously controlled by two complementary objectives\footnote{Alternatively, the \ac{EFE} admits a decomposition into epistemic and pragmatic value as: 
$
G(\tau, \pi_o) = -\mathbb{E}_{q(o,s|\pi_o)}\big[\ln q(s|o,\pi_o) - \ln q(s|\pi_o)\big] - \mathbb{E}_{q(o|\pi_o)}\big[\ln p(o|C)\big], 
$
where the first term is the epistemic value, capturing the expected information gain as the log ratio between posterior and prior beliefs about states~\cite{smith2022step}. Since it is subtracted, minimizing \ac{EFE} drives the agent to maximize information gain, endowing it with an intrinsically information-seeking nature that compels it to resolve uncertainty before exploiting preferred outcomes.}. 
The first objective is goal-directed, where the \ac{AI} agent favors policies whose actions cause expected observations to align with its preferences. Thereby, the \ac{EFE} penalizes any deviation between what the \ac{AI} agent expects and what it prefers. The second objective minimizes uncertainty, where the \ac{AI} agent is intrinsically motivated to seek states that reduce ambiguity about its world model, naturally balancing exploitation of known preferred states with exploration of uncertain ones.

After connecting the \ac{EFE} to the \ac{POMDP} and \ac{VMP} in Section III, it is now possible to write down the \ac{EFE} from policy $\pi_o$ as the sum of the individual \ac{EFE} over future time instants $t< \tau \leq T$, as follows:
\begin{align}
\label{EFE_0}
    G(\pi_o) =\sum_{\tau=t+1}^T G (\tau, \pi_o) 
    &= \sum_{\tau = t+1}^{T} D_{\mathrm{KL}}\big[ q(o_{\tau} \mid  \pi_o) \,\|\, p(o_{\tau} \mid C) \big] + \mathbb{E}_{q(s_\tau\mid\pi_o)} H\Big[p(o_\tau\mid s_\tau)\Big] \nonumber\\
    &= \sum_{\tau = t+1}^{T}  \boldsymbol{A} \boldsymbol{s}_{\pi_o,\tau} \cdot 
    \Big( \ln \boldsymbol{A} \boldsymbol{s}_{\pi_o,\tau} \;-\; \ln \boldsymbol{c} \Big)
    \;-\; \operatorname{diag}\!\big( \boldsymbol{A}^{T} \ln \boldsymbol{A} \big) \cdot \boldsymbol{s}_{\pi_o,\tau},
\end{align}
where $\boldsymbol{c}$ encodes the prior preferences of the \ac{AI} agent embedded within $\mathcal{W}$. Furthermore, without loss of generality, the \ac{EFE} from the alternative policies  $G(\pi)$ can be determined by replacing $\pi_o$ in~\eqref{EFE_0} with $\pi \in \Pi$. Nevertheless, ensuring an adequate assessment between policies requires that all policies $\pi \in \Pi$ start from the same state at time $t$, i.e.,  $q(s_{t}|\pi_o) = q(s_{t}|\pi)$. To evaluate the policies that are likely to minimize \ac{EFE}, we consider updating their posterior probabilities $q(\pi)$ according to the following corollary.

\begin{corollary}
Given the world model $\mathcal{W}(s_{1:T}, o_{1:T}, \pi)$, the variational update for policies $q(\pi)$ satisfies $\ln q(\pi) \;\propto\; - G(\pi), \forall \pi \in \Pi$.
\begin{IEEEproof}
    See Appendix C.
\end{IEEEproof}
\label{corollary}
\end{corollary}

Unlike conventional policy selection which is fixed after training the update 
$\ln q(\pi) \propto -G(\pi)$ in Corollary~\ref{corollary} is recomputed at test time through variational inference 
over the world model $\mathcal{W}$. This dynamically re-weights policies in response to surprise under the current, unforeseen environmental conditions. This is the inference-driven mechanism underlying test-time scaling, where reasoning can contextualize policy selection of \ac{AI} in favor of minimizing \ac{EFE}.
To oprerationalise this scaling process, we introduce $\boldsymbol{\pi}\triangleq (\pi_1, \pi_2, \ldots, \pi_{|\Pi|})$ to capture the \ac{EFE} of policies $\pi_{i} \in \Pi$. Then, we normalize $\ln q^{*}(\boldsymbol{\pi})$ to become a probability distribution while weighing it with the precision $\gamma$ as $q^{*}(\boldsymbol{\pi}) =  \sigma \big(- \gamma \boldsymbol{G}\big)$, where $\boldsymbol{G} \triangleq [G(\pi_1),  G(\pi_2), \ldots, G(\pi_{|\Pi|})]$. 
Accordingly, we can approximate the term $\exp \left(-\gamma  Q_{\mathscr{S}}(a, s_t) \right)$ in~\eqref{scaling} with
$\psi (a \mid s_t, \boldsymbol{\pi}) =  \sum_{\pi \in \Pi} \delta(a, \pi(t)) q^{*}(\pi), \forall a \in \mathcal{A}$. Then, the network shares this inferred policy $\psi (a \mid s_t, \boldsymbol{\pi})$ with the \ac{AI} agent to scale the policy $\pi_o$, as stated in \eqref{scaling}.

Thus, the \ac{AI} agent will now perform a belief update from its prior policy $\pi_o(a,s_t)$ to acquire a posterior $\pi_o' (a \mid s_t)$ by considering the likelihood that actions minimize \ac{EFE} as scored in $\psi (a \mid s_t, \boldsymbol{\pi})$. Nevertheless, this process should still incorporate the level of surprise that modulates the degree to which beliefs are updated. In addition, it can still be challenging to normalize $\pi_o' (a \mid s_t)$ due to the intractability of the marginal likelihood of future observations, as explained in Section II. To overcome this intractable inference problem, we address it from an equivalent perspective that views the \ac{AI} agent as a random dynamical system in interaction with its environment.
Using state flow methods from statistical physics, we can then analyze the behavior of this system as it maintains \ac{NESS} to survive. In particular, this flow can describe how the policy $\pi_o$ must update based on $\psi (a \mid s_t, \boldsymbol{\pi})$. 
Effectively, we will next show how this approach provides a brain-inspired solution to the scale the policy $\pi_o$ into $\pi_o'$ at test time.

From a physics standpoint, an \ac{AI} agent that generalizes in the physical world can be viewed as a random dynamical system that persists over time\footnote{Although the network is actually the agent performing active inference, it is considered that it is performing it on behalf of the \ac{AI} agent (i.e., autonomous vehicle) that is the central entity that should generalize in the unforeseen scenario here.}. This physical system persists by preserving a statistical boundary that separates it as an entity from its environment (see Fig.~\ref{Markov_Blanket}). In particular, this boundary statistically separates the internal beliefs of the \ac{AI} agent from the external states of the world and is essentially a \emph{Markov blanket}~\cite{koller2009probabilistic}. For our system, a Markov blanket can be defined as follows.

\begin{figure}
	\centering
    \includegraphics[width=\columnwidth]{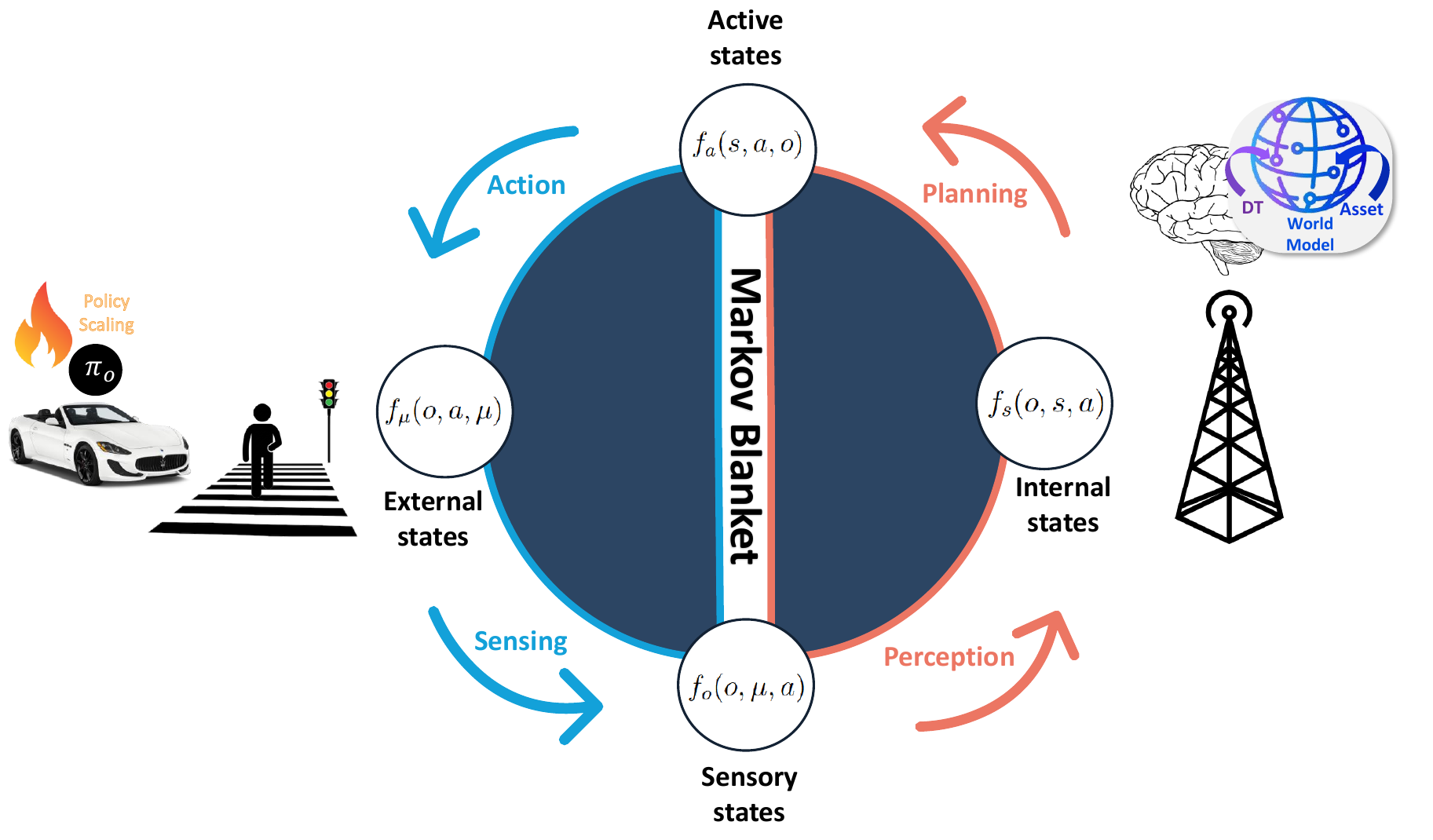}
	\caption{\small{Illustration of the Markov Blanket that statistically separates the \ac{AI} agent from its external environment.}}
	\label{Markov_Blanket}
	\vspace{-0.7cm}
\end{figure}

\begin{definition}
A set of states $b= \{o \in \mathcal{O},a \in \mathcal{A}\}$ forms a \emph{Markov blanket} separating the states of the real-world environment $\mu \in \mathcal{U}$ (i.e., external states) and the beliefs about the states $s \in \mathcal{S}$ (i.e., internal states) iff they are conditionally independent given $b$, i.e., $\mu \perp \!\!\! \perp s \mid b$.
This conditional independence can be expressed as~\cite{da2021bayesian}:
\begin{equation*}
    p(s, \mu \mid b) = p(s \mid b)\,p(\mu \mid b)  \;\;\Longleftrightarrow\;\; p(s \mid b, \mu) = p(s \mid b) 
    \;\;\Longleftrightarrow\;\;  p(\mu \mid b, s) = p(\mu \mid b).
\end{equation*}
\end{definition}

Accordingly, the \ac{AI} agent persists (i.e., survives) in the world by ensuring that this Markov blanket is always maintained. Consequently, this implies that the states of this dynamical system must flow so as to preserve this Markov blanket over time. To formalise this, it is necessary to capture the flow of these states. Here, the state equations of this random dynamical system can be described according to the following differential equation and determined by the corresponding flow (i.e., drift)~$f$: 
\begin{equation}
\dot{x} = f(x) + \omega 
\qquad 
; \qquad f(x) =
\begin{cases}
\begin{bmatrix}
f_{\mu}(o, a, \mu)\\
f_{o}(o, \mu, a) \\
f_{s}(o, s, a) \\
f_{a}(s, a, o)
\end{bmatrix},
\end{cases}
\end{equation}
where $x= \{\mu, o, s, a\}$, $\omega \sim \mathcal{N}(0, \Gamma)$ is a symmetric positive-definite diffusion (i.e., covariance) matrix of the random noise fluctuations from the environment, and $\Gamma$ is the tensor representing half the amplitude of these fluctuations~\cite{palacios2020markov}. Accordingly, the dynamics of this random dynamical system can be formalized in terms of the \emph{Fokker-Planck equation} which governs the temporal evolution of the probability density of the states $p(x,t)$, as follows\footnote{We note that at this moment we can alternatively refer to the Master Equation  to describe the dynamics of the probability density, being a special case of the Fokker-Planck equation~\cite{kaniadakis2018nonlinear}. Thus, the same conclusions can be reached.}:
\begin{equation}
\frac{\partial p(x,t)}{\partial t}  
= - \nabla \cdot \big[ f(x) \, p(x,t) \big] 
+ \nabla \cdot \big[ \Gamma \, \nabla p(x,t) \big].
\label{FPE}
\end{equation}

Here, \eqref{FPE} describes how the probability density over the system's 
states evolves over time under the drift $f(x)$ and diffusion $\Gamma$. Since this density encodes the system's beliefs over its states (i.e., sensory, internal, and active states), the two are directly connected through the same underlying dynamics. Specifically, the drift $f(x)$ plays the role of the belief update rules in~\eqref{perception} and \eqref{scaling}, steering the system toward states that consistently maximize model evidence. This stands in contrast to RL, which optimizes the policy through reward maximization without accounting for the \ac{NESS} structure of the system that needs to be maintained at test time. Thereby, this leaves the underlying stochastic dynamics unmodeled while missing the connection to belief updating of policies in unforeseen scenario at test time.

Since the \ac{AI} agent works to maintain its survival whereby it self-organizes towards its preferences, this means that the dynamical system must conserve its \ac{NESS}. In this case, the probability flow is balanced while having a continuous flux, yet the probability density $p(x,t)$ remains invariant to achieve a steady state. Henceforth, the probability density $ p(x,t)$ converges into an ergodic density $p^*(x)$. Thus, the preferred states in $p^*(x)$ act as a pullback attractor for this dynamical system, encoding the set points where the \ac{AI} agent survives. Effectively, \eqref{FPE} can then be simplified into a stationary density setting. In this case, we can determine the flow of the states in this \ac{NESS} which ensure the persistence of a Markov blanket according to the following lemma.

\begin{lemma}
For a stochastic dynamical system that admits a Markov blanket in \ac{NESS}, the drift $f(x)$ of its states can be decomposed  into:
\begin{equation}
f(x) = -\Gamma \nabla L(x) + R \nabla L(x) = (\Gamma - R) \nabla \ln p^{*}(x),
\end{equation}
where $L(x) = -\ln p^*(x)$ is the Lagrangian (i.e., potential function) associated with the steady-state density $p^*(x)$, and $R$ is an antisymmetric solenoidal operator.
    \begin{IEEEproof}
        See Appendix D.
    \end{IEEEproof}
\label{Helmholtz}
\end{lemma}


From Lemma~\ref{Helmholtz}, we observe that the states of any random dynamical system that preserves a Markov blanket will flow towards their preferences, and thereby, minimize surprise~\cite{friston2014cognitive}. Based on Lemma~\ref{surpise}, the states of this random dynamical system equivalently flow towards minimizing their \ac{VFE} which upper bounds surprise. In particular, the flow $f_{a}(s, a, o)$ must drive the active states $a \in \mathcal{A}$ towards these regions of low free energy, i.e., $f_{a}(s, a, o) = (\Gamma - R) \nabla_a \ln p(s, a, o) \leq (R - \Gamma) \nabla_a F(s, a, o)$.
In the case of a surprise, the active states (which correspond to the actions taken by the agent) are therefore flowing from $\pi_o(a,s_t)$ to $\psi (a \mid s_t, \boldsymbol{\pi})$ to minimize free energy $F(s, a, o)$.
Nevertheless, this transition from $\pi_o(a \mid s_t)$ to $\psi (a \mid s_t, \boldsymbol{\pi})$ should still be modulated according to the magnitude of the surprise $\mathscr{S} (t, \pi_o)$. Effectively, this flow is then equivalent to performing a gradient descent over the free energy landscape while considering the impact of surprise. Here, the free energy in terms of active states is $F(s, a, o) \propto D_{\mathrm{KL}}\!\left(\pi_o' (a\mid s_t) \,\|\, \psi (a \mid s_t, \boldsymbol{\pi})\right)$.
Hence, minimizing $F(s,a ,o)$ is proportional to minimizing the \ac{KL} divergence between the $\pi_o'(a \mid s_t)$ and $\psi (a \mid s_t, \boldsymbol{\pi})$. Following a similar derivation to~\eqref{gradient}, we define the action prediction error as $\epsilon_{\pi_{o}, \psi} =  \nabla_a F(s, a, o) = \ln \psi (a \mid s_t, \boldsymbol{\pi}) - \ln \pi_o' (a \mid s_t)$. 
Here, the posterior beliefs about the actions are obtained through a gradient descent as in~\eqref{iterative}. Starting initially from $\pi_o'(a \mid s_t)   = \pi_o (a \mid s_t) $, the \emph{scaled policy} $\pi'(a \mid s_t)$ can then be reached through an iterative update that consists of the following steps:
\begin{align}
\label{gradient_descent_equations}
 \left\{
\begin{aligned}
    & \epsilon_{\pi_{o}, \psi} \leftarrow  \;  \ln \psi (a \mid s_t, \boldsymbol{\pi}) - \ln \pi_o' (a\mid s_t),  \\
    & \ln \pi_o' (a\mid s_t)  \leftarrow \ln \pi_o' (a\mid s_t) + \epsilon_{\pi_{o}, \psi}, \\ 
    &\pi_o' (a\mid s_t)  \leftarrow \sigma \big(\ln \pi_o' (a \mid s_t)\big).
\end{aligned}
 \right . 
 \end{align}

\begin{figure}
	\centering
    \includegraphics[width=0.85\columnwidth]{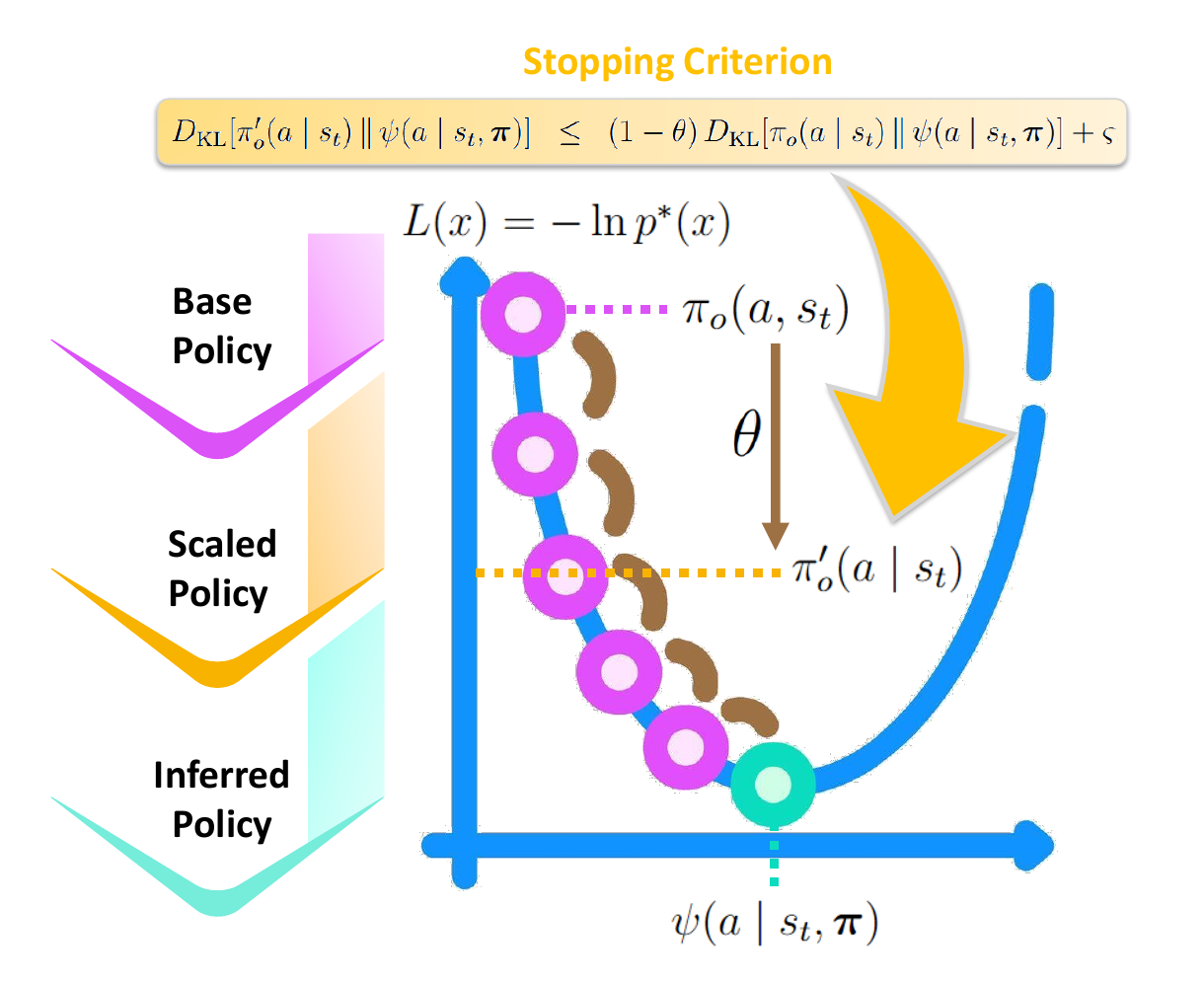}
	\caption{\small{Illustration of the gradient descent operation over the Lagrangian to capture the scaled policy as derived in~\eqref{gradient_descent_equations} and \eqref{stopping_criterion}.}}
	\label{Gradient_Descent}
	\vspace{-0.7cm}
\end{figure}

As this gradient descent operation corresponds to the soft Bayesian update in~\eqref{scaling}, the scaled policy $\pi'_o (a\mid s_t)$ is reached once a stopping criterion is met. This stopping criterion will incorporate the effect of surprise $\mathscr{S} (t, \pi_o)$, whereby the gradient descent stops iterating when the remaining \ac{KL} divergence has been reduced to the surprise-dependent target $\theta$.
This target can be defined through an exponential normalization over surprise, i.e.,  $\theta (\mathscr{S}) = 1-\exp(-\Omega\mathscr{S})$:
\begin{equation}
    D_{\mathrm{KL}}\!\left[\pi_o' (a\mid s_t) \,\|\, \psi (a \mid s_t, \boldsymbol{\pi})\right] 
\;\;\le\;\; (1 - \theta)\, D_{\mathrm{KL}}\!\left[\pi_o (a\mid s_t) \,\|\, \psi (a \mid s_t, \boldsymbol{\pi})\right] + \varsigma,
\label{stopping_criterion}
\end{equation}
where $0 \leq \varsigma <\theta$ is a tolerance error for numerical precision and $\Omega > 0$ is a saturation parameter.
As shown in Fig.~\ref{Gradient_Descent}, this update implements policy scaling via gradient descent, where the prediction error $\epsilon_{\pi_{o}, \psi}$ drives the policy $\pi_o'$ toward the posterior $\psi(a \mid s_t, \boldsymbol{\pi})$ through iterative belief refinement (i.e., inference). This constitutes the operational realization of the test-time scaling law derived in~\eqref{scaling}. It represents the central contribution of the paper, which provides a principled inference-driven mechanism for dynamically scaling the policy at test time in response to prediction errors, without requiring retraining.

However, scaling the policy at test time is insufficient for long-term survival, as it only provides an instant solution for generalization. Over time, the agent must also consolidate these resolved instances into its world model $\mathcal{W}$ and base policy $\pi_o$ to retain this knowledge for future encounters.
\emph{This natural motivation for survival drives the \ac{AI} agent to learn at test time.}
In the following section, we derive how this learning proceeds, enabling the \ac{AI} agent to generalize beyond its training distribution and reduce future surprise by retaining knowledge of previously unseen scenarios.













\section{Learning at Test Time: Updating the World Model and Policy via Inference}
\label{Learning via Inference}

At the end of the time interval $T$, the surprising instances that have been resolved represent an opportunity for the \ac{AI} agent to \emph{``learn by experience"} from the world~\cite{silver2025welcome}. 
It is important to distinguish this from test-time inference. On the one hand, the latter operates on a fast timescale, scaling the policy through a belief update driven by surprise without modifying any model parameters. On the other hand, learning operates on a slower timescale by performing parameter updates to the world model $\mathcal{W}$ and policy $\pi_o$ to permanently incorporate the resolved instances.
The goal of this learning is to incorporate these instances into the world model $\mathcal{W}$ and policy $\pi_o$ such that they become more predictable by the agent and \emph{less surprising in the future}. In other words, the \ac{AI} agent becomes less susceptible to similar situations when encountered again, whereby action selection can progressively shift from being dominated by deliberative reasoning to a fast, feed-forward decision. This update process is a natural consequence of the survival objective that intrinsically drives \ac{AI} agents to continuously learn at test time.  

Unlike perception, planning, and action that require inference to perform instantaneous belief updates, learning is associated with synaptic plasticity dynamics (i.e., Hebbian plasticity) that usually occur at a lower pace and over a longer time period~\cite{da2020active}. This is due to the fact that learning requires an \emph{accumulation of evidence about the world} to permanently update parameters. On the one hand, the parameters of the likelihood model $\boldsymbol{A}$ and transition model $\boldsymbol{B}$ defined in Section~\ref{Perception} would need to reflect the effect of the surprising pairs of states and observations. On the other hand, the policy $\pi_o$ needs to consolidate the actions sampled from its scaled version $\pi_o'$. Thus, learning can be cast as a \emph{Bayesian update} of these parameters, given observations, inferred hidden states, and scaled actions. Consequently, learning can then be modeled as Bayesian inference over these parameters, in attempt to minimize the prediction errors (i.e., \ac{VFE}). Henceforth, in this section, we will explain how learning via inference enhances its experience through continuous interaction with the world. This learning stage closes the loop between test-time scaling and training-time adaptation~\cite{dupoux2026ai}. In particular, it converts the inference-driven policy updates -- accumulated during deployment -- into permanent parameter updates of the world model $\mathcal{W}$ and policy $\pi_o$, thereby consolidating test-time experience into long-term knowledge.

\subsection{Updating the World Model: Observation Model $\boldsymbol{A}$}
To incorporate the learning process within inference, it is necessary to expand the world model $\mathcal{W}(o_{1:T}, s_{1:T}, \pi)$ in~\eqref{world model} to include the likelihood model $\boldsymbol{A}$ and transition model $\boldsymbol{B}$, as follows\footnote{Subscripts from $\boldsymbol{B}_{\pi, \tau}$ are dropped for simplicity, while noting that $\pi$ is considered to be the base policy $\pi_o$ in this case.}:
\begin{equation}
p(o_{1:T}, s_{1:T}, \pi, \boldsymbol{A}, \boldsymbol{B}) 
=   p(\pi) p( \boldsymbol{A}) p(\boldsymbol{B}) p(s_{1}) \prod_{\tau = 2} ^{T} p(s_\tau \mid s_{\tau-1},\pi, \boldsymbol{B}) \prod_{\tau =1}^{T} p(o_\tau \mid s_\tau, \boldsymbol{A}).
\label{new_world_model}
\end{equation}

By following steps similar to \eqref{MFA_0} -- \eqref{Euler-Lagrange}, we aim to find the optimal posterior $q^*(\boldsymbol{A})$ that minimizes the \ac{VFE}. Here, the \ac{VFE} is captured in its expanded form as:
\begin{equation}
F[q(s_{1:T}, \pi_o, \boldsymbol{A}, \boldsymbol{B})] 
= \mathbb{E}_{q(s_{1:T},  \pi_o, \boldsymbol{A}, \boldsymbol{B})}\!\left[ \ln q(s_{1:T},  \pi_o, \boldsymbol{A}, \boldsymbol{B}) - \ln p(s_{1:T},  o_{1:T}, \pi_o, \boldsymbol{A}, \boldsymbol{B}) \right].
\label{new_VFE}
\end{equation}

\noindent Then, we consider minimizing  $F[q(s_{1:T}, \pi_o, \boldsymbol{A}, \boldsymbol{B})]$ with respect to $\boldsymbol{A}$ under a mean field approximation, i.e., 
$q(s_{1:T}, \pi_o, \boldsymbol{A}, \boldsymbol{B}) \approx q(\pi_o) q(\boldsymbol{A}) q(\boldsymbol{B})\prod_{\tau = 1 }^{T} q(s_{\tau} | \, \pi_o).$
This approximation assumes that the parameters $\boldsymbol{A}$, $\boldsymbol{B}$, and the policy $\pi_o$ are mutually independent, and that the states of the world factorize across time given the policy $\pi_o$. While this renders inference tractable -- by decomposing the joint distribution into independent factors that can be optimized separately -- it neglects conditional dependencies between parameters and states. Hence, it represents a tradeoff between computational tractability and approximation accuracy.
Accordingly, we can rewrite the \ac{VFE} in~\eqref{new_VFE} as follows:
\begin{align}
F[q(\boldsymbol{A})] 
& \propto  \mathbb{E}_{q(\boldsymbol{A})} [\ln q(\boldsymbol{A})] - \mathbb{E}_{q(s_{1:T}, \pi_o, \boldsymbol{A},\boldsymbol{B})}[ \ln p(s_{1:T},  o_{1:T}, \pi_o, \boldsymbol{A}, \boldsymbol{B})] \nonumber\\
& \quad = \mathbb{E}_{q(\boldsymbol{A})} \Big[ \ln q(\boldsymbol{A}) - \underbrace{\mathbb{E}_{q(s_{1:T}, \pi_o, \boldsymbol{B})}[\ln p(s_{1:T},  o_{1:T}, \pi_o, \boldsymbol{A}, \boldsymbol{B})]}_{\ln q^*(A)} \Big] \nonumber\\
& \quad = D_{\mathrm{KL}}\big[ q(\boldsymbol{A})   \,\|\,  q^*(\boldsymbol{A}) \big].
\label{KL-Divergence-A}
\end{align}

\noindent Therefore, the \ac{VFE} can be minimized with respect to $\boldsymbol{A}$ through the following Euler-Lagrange step solution that minimizes the \ac{KL} divergence in~\eqref{KL-Divergence-A} :
\begin{equation}
\ln q^*(\boldsymbol{A}) 
= \mathbb{E}_{q(s_{1:T}, \pi_o, \boldsymbol{B})}[ \ln p(s_{1:T},  o_{1:T}, \pi_o, \boldsymbol{A}, \boldsymbol{B})] 
\label{Euler-Lagrange-A}
\end{equation}
Considering that the terms that depend on $\boldsymbol{A}$ from~\eqref{new_world_model} are the only ones that remain as variables in~\eqref{Euler-Lagrange-A}, we can further simply ~\eqref{Euler-Lagrange-A} that yields:
\begin{equation}
\underbrace{\ln q^*(\boldsymbol{A})}_{\mathrm{posterior}} \propto  \underbrace{\ln p(\boldsymbol{A})}_{\mathrm{prior}} + \underbrace{\sum_{\tau=1}^{T} \mathbb{E}_{q(s_\tau)}[\ln p(o_\tau \mid s_\tau, \boldsymbol{A})]}_{\mathrm{likelihood}}.  
\label{posterior_A}
\end{equation}

As discussed earlier in Section~\ref{Perception}, the likelihood $\boldsymbol{A}$ is formally defined as a categorical distribution, i.,e., \( p(o_\tau \mid s_\tau, \boldsymbol{A}) = \mathrm{Cat}(o_\tau | \boldsymbol{A} \, s_\tau) \). Here, each column \( \boldsymbol{A}_{\cdot j} \) represents a categorical distribution over possible observations $o_\tau$ conditioned on the state \( s_\tau \).
With \(p(o_\tau = i \mid s_\tau = j, \boldsymbol{A}) = A_{ij}\) to index the entries of $\boldsymbol{A}$, the expected log-likelihood in~\eqref{posterior_A} can be modified to become:
\begin{equation}
\mathbb{E}_{q(s_\tau)}[\ln p(o_\tau \mid s_\tau, \boldsymbol{A})] = 
\sum_{j=1}^{|\mathcal{S}|}q(s_\tau = j)\sum_{i=1}^{|\mathcal{O}|}\mathbf{1}_{\{o_\tau=i\}}\ln A_{ij},
\label{likelihood_A}
\end{equation}
where $\mathbf{1}_{\{o_\tau=i\}}$ is an indicator function that equals $1$ when the observed outcome $o_{\tau}$ corresponds to \( i \). 

Subsequently, each column $\boldsymbol{A}_{\cdot j}$ of the likelihood matrix $\boldsymbol{A}$ is parameterized by a Dirichlet distribution, i.e., \(p(\boldsymbol{A}_{\cdot j}) = \mathrm{Dir}(\boldsymbol
{A}_{\cdot j} \mid \boldsymbol{\alpha}_{\cdot j}) \) with concentration parameters $\boldsymbol{\alpha}_{\cdot j} = [\alpha_{1j}, \ldots, \alpha_{|\mathcal{O}|j}]^{\top}$. Here, each $\alpha_{ij} \in \mathbb{R}_{>0}$  denotes a prior pseudo-count over observation $i$ given state \( j \).
Assuming that the columns of $ \boldsymbol{A}$ are independent, \(p(\boldsymbol{A})\) can be modeled as a product of Dirichlet distributions, as follows:
\begin{equation}
p(\boldsymbol{A}) = \prod_{j=1}^{|\mathcal{S}|}\mathrm{Dir}(\boldsymbol{A}_{\cdot j} \mid \boldsymbol{\alpha}_{\cdot j})
= \prod_{j=1}^{|\mathcal{S}|}\frac{1}{B(\boldsymbol{\alpha}_{\cdot j})}\prod_{i=1}^{|\mathcal{O}|}A_{ij}^{\alpha_{ij}-1},
\label{prior_A}
\end{equation}
where $B(\boldsymbol{\alpha}_{\cdot j}) =  \dfrac{\prod_{i=1}^{|\mathcal{O}|} \Psi(\alpha_{ij})}{
\Psi\!(\sum_{i=1}^{|\mathcal{O}|} \alpha_{ij})}$ is the multivariate Beta function and $\Psi(\cdot)$ is the Gamma function. 
Then, substituting~\eqref{likelihood_A} and~\eqref{prior_A} back in~\eqref{posterior_A} yields:
\begin{align}
\ln q^*(\boldsymbol{A}) &\propto
 \sum_{i=1}^{|\mathcal{O}|} \sum_{j=1}^{|\mathcal{S}|}
\left[
\alpha_{ij} - 1 + \sum_{\tau=1}^{T}\mathbf{1}_{\{o_\tau=i\}}q(s_\tau=j) 
\right]\ln A_{ij},
\end{align} 
which is recognized as the logarithm of a Dirichlet distribution whose concentration parameters are $\alpha_{ij}^{'} \triangleq \alpha_{ij} + \sum_{\tau=1}^{T}\mathbf{1}_{\{o_\tau=i\}}q(s_\tau=j)$.
Henceforth, this yields a posterior Dirichlet distribution which factorizes into the form:
\begin{equation}
q^*(\boldsymbol{A}) = \prod_{j=1}^{|\mathcal{S}|}
\mathrm{Dir}\!\left(\boldsymbol{A}_{\cdot j} \,\middle|\, 
\boldsymbol{\alpha}_{\cdot j}^{\prime}\right),
\end{equation}
where $\boldsymbol{\alpha}_{\cdot j}^{\prime} \triangleq [\alpha_{1j}', \ldots, \alpha_{|\mathcal{O}|j}']^{\top}$ represent the updated concentration parameters over each column $\boldsymbol{A}_{\cdot j}$. Intuitively, learning thus corresponds to updating the prior Dirichlet parameters by accumulating expected co-occurrences between sensory observations $i$ and inferred hidden states~$j$. Consequently, this increases the corresponding concentration parameter \( \alpha_{i j} \). Equivalently, learning can be viewed as updating the concentration parameters of the Dirichlet distribution by accumulating soft counts of outcome--state co-occurrences over time, weighted by the inferred posterior over hidden states $q^*(s_\tau)$.

\subsection{Updating the World Model: Transition Model $\boldsymbol{B}$}
Similar to the derivation for the likelihood parameters $\boldsymbol{A}$ in~\eqref{KL-Divergence-A} - \eqref{posterior_A}, we can now derive the posterior over the transition probabilities $\boldsymbol{B}$. 
To minimize the \ac{VFE} in~\eqref{new_VFE} with respect to $\boldsymbol{B}$, the terms that depend on $\boldsymbol{B}$ are only those involving $p(\boldsymbol{B})$ and $p(s_{\tau} \mid s_{\tau-1}, \boldsymbol{B})$. Hence, we can derive the Euler-Lagrange solution for $\boldsymbol{B}$ as:
\begin{equation}
\ln q^*(\boldsymbol{B}) \propto  \ln p(\boldsymbol{B}) + \sum_{\tau=2}^{T} \mathbb{E}_{q(s_\tau, s_{\tau-1})}[\ln p(s_\tau \mid s_{\tau-1}, \boldsymbol{B})].   
\label{posterior_B}
\end{equation}
Here, the policy $\pi$ is dropped from log likelihood term in~\eqref{posterior_B} because it is treated as a fixed parameter, having the \ac{AI} agent committed to $\pi_o$. 
Moreover, the transition model is defined as categorical distribution such that $p(s_\tau \mid s_{\tau-1}, \boldsymbol{B}) = \mathrm{Cat}(s_\tau \mid \boldsymbol{B} \, s_{\tau-1})$.
In particular, each column $\boldsymbol{B}_{\cdot j}$ defines a categorical distribution over  states $s_{\tau}$ given the previous state $s_{\tau-1}$, indexed with entries $p(s_{\tau} =i \mid s_{\tau-1} = j, \boldsymbol{B}) = B_{ij}$.
Therefore, the expected log likelihood term in~\eqref{posterior_B} becomes:
\begin{equation}
\mathbb{E}_{q(s_\tau, s_{\tau-1})} \!\left[ \ln p(s_\tau \mid s_{\tau-1}, \boldsymbol{B}) \right]
= \sum_{j=1}^{|\mathcal{S}|} \sum_{i=1}^{|\mathcal{S}|} q(s_{\tau-1} = j, s_\tau = i) \ln B_{ij}.
\end{equation}

Then, we assume a Dirichlet prior over each column of the transition matrix, i.e.,
$p(\boldsymbol{B}_{\cdot j}) = \mathrm{Dir}(\boldsymbol{B}_{\cdot j} \mid \boldsymbol{\beta}_{\cdot j})$ having concentration parameters $\boldsymbol{\beta}_{\cdot j}= [\beta_{1j}, \ldots, \beta_{|\mathcal{S}|j}]^{\top}$.
With $p(\boldsymbol{B})$ having a similar form to~\eqref{prior_A}, this transforms~\eqref{posterior_B} into:
\begin{equation}
\ln q^*(\boldsymbol{B}) \propto
\sum_{j=1}^{S} \sum_{i=1}^{S}
\left( \beta_{ij} - 1 + \sum_{\tau=2}^{T} q(s_{\tau-1} = j, s_\tau = i) \right) \ln B_{ij},
\end{equation}
which is recognized as the logarithm of a Dirichlet distribution whose concentration parameters are $\beta_{ij}^{'} \triangleq \beta_{ij} + \sum_{\tau=2}^{T} q(s_{\tau-1} = j, s_\tau = i)$. Thus, this yields a posterior Dirichlet distribution for  $\boldsymbol{B}$ with concentration parameters $\boldsymbol{\beta}_{\cdot j}' = [\beta_{1j}', \ldots, \beta_{|\mathcal{S}|j}']^{\top}$ over each column $j$. Henceforth,  $q^*(\boldsymbol{B})$ can be factorized into the following form:
\begin{equation}
q^*(\boldsymbol{B}) = \prod_{j=1}^{|\mathcal{S}|}
\mathrm{Dir} (\boldsymbol{B}_{\cdot j} \mid \, \boldsymbol{\beta}_{\cdot j}').
\end{equation}

Next, we will derive how the policy $\pi_o$ can be updated to incorporate the resolved, surprising situations by consolidating the actions acquired through reasoning from the scaled policy $\pi_o'$.


\subsection{Updating the Policy $\pi_o$}
As the likelihood and transition parameters (i.e., $\boldsymbol{A}$ and $\boldsymbol{B}$, respectively) of the world model are updated to better anticipate the surprising situations, the policy $\pi_o$ must also adapt to incorporate the actions that successfully resolve these new situations. This consolidation is a key conceptual step, whereby reasoning-derived actions that proved effective in unforeseen scenarios are gradually absorbed into the habitual base policy $\pi_o$. Effectively, this converts the action discovered through deliberative reasoning into effortless, feed-forward behavior. Thus, \emph{learning closes the loop between test-time scaling and long-term policy adaptation.}
This learning update can be modeled as an inference process which aims to find the optimal posterior $q^*(\pi_o)$ that minimizes the \ac{VFE} in~\eqref{new_VFE}. 
Similar to~\eqref{Euler-Lagrange-A}, the Euler-Lagrange step that minimizes~\eqref{new_VFE} with respect to $\pi_o$ is found to be:
\begin{equation}
\ln q^*(\pi_o)
= \mathbb{E}_{q(s_{1:T}, \boldsymbol{A}, \boldsymbol{B}))}[\ln  p(o_{1:T}, s_{1:T}, \pi_o, \boldsymbol{A}, \boldsymbol{B})].
\label{Euler-Lagrange-pi_o}
\end{equation}
From~\eqref{new_world_model}, the only terms that remain in~\eqref{Euler-Lagrange-pi_o} are those that depend on $\pi_o$, namely $p(\pi_o)$ and $p(s_\tau \mid s_{\tau-1}, \pi_o, \boldsymbol{B})$. Hence, we can simplify~\eqref{Euler-Lagrange-pi_o} to become: 
\begin{equation}
   \ln q^*(\pi_o)
\propto \ln p(\pi_o)
+ 
 \mathbb{E}_{q(s_{1:T},  \boldsymbol{B})} \left[ \sum_{\tau=2}^{T} \ln p(s_\tau \mid s_{\tau-1}, \pi_o, \boldsymbol{B}) \right].
 \label{Euler-Lagrange-pi_o-simplified}
\end{equation}
To explicitly expose the dependence on $\pi_o$ in~\eqref{Euler-Lagrange-pi_o-simplified}, we separate the process of action selection under policy $\pi_o$ from state transitions driven by action $a_\tau$, such that\footnote{The policy $\pi_o$ encodes the probability of selecting each action in a given state, while the transition model $p(s_{\tau} \mid s_{\tau-1}, a_{\tau}, \boldsymbol{B})$ captures how those actions influence the subsequent states.}:
\begin{equation}
p(s_{\tau} \mid s_{\tau-1}, \pi_o,  \boldsymbol{B})
= \sum_{a_{\tau} \in \mathcal{A}} \underbrace{p(s_{\tau} \mid s_{\tau-1}, a_{\tau},  \boldsymbol{B})}_{\mathrm{state\,transition}} \, \underbrace{p(a_{\tau} \mid s_{\tau-1}, \pi_o)}_{\mathrm{action\,selection}}.
\label{decomposition}
\end{equation}

Given that the sequence of actions $a_{1:T}$ have been selected and executed during time $T$, the summation in~\eqref{decomposition} can be reduced to a single instance. 
In addition, the terms that depend on $\pi_o$ in~\eqref{decomposition} effectively reduce (i.e., $p(s_{\tau} \mid s_{\tau-1}, a_{\tau},  \boldsymbol{B})$ does not depend on $\pi_o$) to the policy likelihoods $p(a_\tau \mid s_{\tau-1}, \pi_o)$. Thus, \eqref{Euler-Lagrange-pi_o-simplified} can be expressed as:
\begin{align}
\ln q^*(\pi_o) 
& \propto \ln p(\pi_o) + \mathbb{E}_{q(s_{1:T}, a_{1:T})} \left[\sum_{\tau=2}^{T} \ln  p(s_{\tau} \mid s_{\tau-1}, a_{\tau},  \boldsymbol{B}) p(a_\tau \mid s_{\tau-1}, \pi_o)\right] \nonumber \\
& \quad = \ln p(\pi_o) + \mathbb{E}_{q(s_{1:T}, a_{1:T})}
\left[\sum_{\tau=2}^{T} \ln p(a_\tau \mid s_{\tau-1}, \pi_o)\right],
\label{Reduced_Lagrangian_pi}
\end{align}
where the expectation $\mathbb{E}_{q(\cdot)}$ is taken over the actions $a_{1:T}$ that are included into the world model in~\eqref{new_world_model}. 
Then, the policy likelihood\footnote{The policy likelihood $p(a_\tau \mid s_{\tau-1}, \pi_o)$ in the \ac{PFC} has an equivalent role to the policy $\pi_o$ in the \ac{BG}.} is defined as a categorical distribution over actions given the previous state, i.e., $p(a_\tau \mid s_{\tau-1}, \pi_o) = \mathrm{Cat}(a_\tau \mid \Tilde{\boldsymbol{\Pi}} s_{\tau-1})$. Here, $\Tilde{\boldsymbol{\Pi}}$ is the likelihood matrix of choosing action $a_\tau$ state in $s_{\tau-1}$ under policy $\pi_o$.
Thus, each column $\Tilde{\boldsymbol{\Pi}}_{\cdot j}$ defines a categorical distribution over actions given previous states, indexed with the entries $p(a_\tau =i \mid s_{\tau-1} = j, \pi_o) = \Tilde{\Pi}_{ij}$. Therefore, the expected policy likelihood term in~\eqref{Reduced_Lagrangian_pi} can be expressed as: 
\begin{equation}
\mathbb{E}_{q(s_{1:T},  a_{1:T})}
\left[\sum_{\tau=2}^{T}\ln p(a_\tau \mid s_{\tau-1}, \pi_o)\right]
= \sum_{\tau=2}^{T} \sum_{j = 1}^{|\mathcal{S}|} \sum_{i = 1}^{|\mathcal{A}|} q(s_{\tau-1}=j)\,\mathbf{1}_{\{a_\tau=i\}} \ln \Tilde{\Pi}_{ij},
\end{equation}
whereby it represents the expected number of times action $a_\tau = i$ was taken in state $s_{\tau-1} = j$, given inferred states. Moreover, we assume a Dirichlet prior over each column $\Tilde{\boldsymbol{\Pi}}_{\cdot j}$ such that $p(\Tilde{\boldsymbol{\Pi}}_{\cdot j}) = \mathrm{Dir}(\Tilde{\boldsymbol{\Pi}}_{\cdot j} \mid \boldsymbol{\nu}_{\cdot j})$ having concentration parameters $\boldsymbol{\nu}_{\cdot j} = [\nu_{1j}, \ldots, \nu_{|\mathcal{A}|j}]^{\top}$. With $p(\pi_o) = p(\Tilde{\boldsymbol{\Pi}}) = \prod_{j=1}^{|S|} \mathrm{Dir}(\Tilde{\boldsymbol{\Pi}}_{\cdot j} \mid \boldsymbol{\nu}_{\cdot j})$, \eqref{Reduced_Lagrangian_pi} transforms into:
\begin{equation}
\ln q^*(\pi_o)
\propto \sum_{j = 1}^{|\mathcal{S}|} \sum_{i = 1}^{|\mathcal{A}|} \big(\nu_{ij} - 1 + \sum_{\tau=2}^{T} q(s_{\tau-1}=j)\,\mathbf{1}_{\{a_\tau=i\}}\big)\ln \Tilde{\Pi}_{ij},
\end{equation}
which corresponds to a logarithmic of a product of Dirichlet distributions with concentration parameters $\nu_{ij}' = \nu_{ij} + \sum_{\tau=2}^{T} q(s_{\tau-1}=j)\,\mathbf{1}_{\{a_\tau=i\}}$. Therefore, this yields a posterior Dirichlet distribution with concentration parameters $\boldsymbol{\nu}_{\cdot j}' = [\nu_{1j}', \ldots, \nu_{|\mathcal{A}|j}']$ over each column $j$, which factorizes into:
\begin{equation}
q^*(\pi_o)
= \prod_{j=1}^{|\mathcal{S}|} \mathrm{Dir}(\Tilde{\boldsymbol{\Pi}}_{\cdot j} \mid \boldsymbol{\nu}_{\cdot j}').
\label{updated_policy}
\end{equation}

\begin{figure}
	\centering
    \includegraphics[width=\columnwidth]{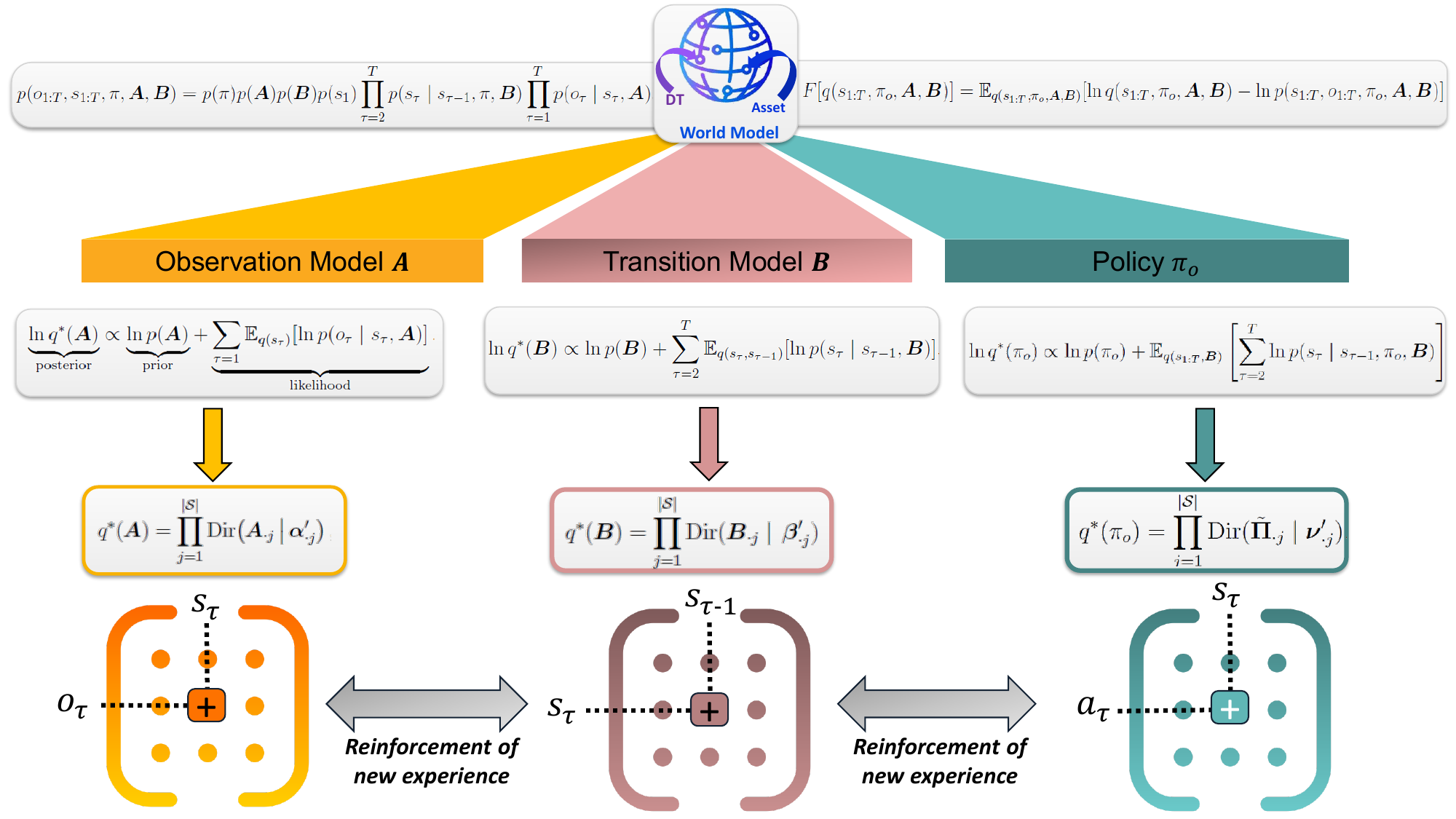}
	\caption{\small{Illustration of the learning process that is modeled as inference to reinforce the new experiences discovered at test time to enable continually learning the policy of the physical \ac{AI} agent and the world model.}}
	\label{Learning}
	\vspace{-0.7cm}
\end{figure}

This showcases how updating the policy $\pi_o$ in the \ac{PFC} corresponds to a Dirichlet accumulation of state--action evidence which \emph{reinforces} this new behavior discovered via reasoning (i.e., \ac{VFE} minimization). In fact, it shows how prior pseudo-counts $\nu_{ij}$ are transformed with empirical, surprising experience into the posterior $\nu_{ij}'$. Subsequently, the inferred posterior policy thus moves from the \ac{PFC} to the \ac{BG} to reflect this update on the policy $\pi_o$. \emph{Remarkably, this observation demonstrates how \acp{DT} play a crucial role in autonomously updating the policy of the physical \ac{AI} agent in the real world}\footnote{Alternatively, in a deep learning scenario, this Bayesian learning update corresponds to a continual learning solution. Therein, the prior policy is updated with new accumulated evidence from the physical world (e.g., via elastic weight consolidation~\cite{hashash2022edge}).}.
While \acp{DT} indeed drive the world model $\mathcal{W}$ to instill reasoning in physical \ac{AI} agents so as to generalize at test time, learning shows that the connection between the \ac{PT} and \acp{DT} must also be used to update the policy $\pi_o$ as a natural consequence of survival.
Clearly, learning exploits the same \ac{PFC}-\ac{BG} connection used during test-time scaling. The resulting policy update in the \ac{PFC} in~\eqref{updated_policy} is then reflected back to the \ac{BG}. This reduces the prediction errors upon future encounters with similar new situations\footnote{Notably, more comprehensive solutions for identifying unforeseen scenarios that do not fall under any state $s_\tau \in \mathcal{S}$ would need to consider methods like Bayesian model expansion as a potential treatment, which is outside the scope of this work.}. In fact, \emph{this learning can compensate for situations that were not essentially covered throughout the training time, but with a higher cost of deliberative reasoning at test time.} 
This learning procedure that reinforces the new experiences and updates the parameters of the world model is summarized in Fig.~\ref{Learning}. 
These nested \ac{VFE} minimizing processes broadly capture the basic nature of how intelligent systems adapt to the changing environments by changing their behavior in a continuous manner to ensure their ``adaptive fitness", where ``fit" can be read as world model fitting, via active inference.


\section{Simulation Results and Analysis}
 
For our simulations, we consider a physical \ac{AI} agent to be an autonomous vehicle that has a driving task objective. 
This experiment follows the example of Fig.~\ref{Passive_to_Active_Inference}.
In particular, our simulation environment models a vehicle approaching an intersection to cross towards its destination goal, given that pedestrians may appear at the traffic light intersection.
Accordingly, the state space $\mathcal{S}$ comprises distance $d \in \{0,1,2,3\}$ that measures the distance to the destination, velocity of the vehicle $v \in \{0,1,2,3\}$, traffic light state $L \in \{0,1\}$ to represent red and green respectively, and pedestrian presence $p \in \{0,1\}$, yielding 64 discrete states. The action space $\mathcal{A}$ consists of three actions that control $v$: \emph{maintain}, \emph{speed up}, and \emph{slow down}. The reward function (i.e., prior preferences) encourages rapid goal approach with a crash penalty for collisions with pedestrians at the intersection at $d=1$. In essence, the vehicular agent learns a distribution where pedestrians appear exclusively at red lights. Accordingly, an unforeseen scenario occurs at test time when pedestrians \emph{jaywalk} i.e., $p=1$ when $L$ is green.
This setup constitutes a controlled out-of-distribution test-time generalization benchmark, where the jaywalking scenario represents a deliberate distributional shift from the training distribution. Hence, this setup allows a rigorous evaluation of the robustness and generalization capabilities of the  \ac{AI} agent under unforeseen conditions. As discussed earlier, this is a meaningful use case and example that can capture the broader problem facing physical \ac{AI} agents at test time and can cover other diverse examples with \ac{AI} agents that are performing different tasks.

In our experiments, we compare our proposed test-time scaling solution with two benchmarks: 
a) Q-learning and b) Bayesian \ac{RL}. In particular, Q-learning serves as model-free approach, while Bayesian \ac{RL} is typically a model-based approach that relies on inference-based planning. 
It is worth noting that all methods share the same training data, reward function, 
and environment assumptions, ensuring a fair comparison that isolates the effect of 
test-time scaling as the differentiating factor between the approaches evaluated.
Our test-time scaling solution builds on Q-learning as its underlying base policy. Both policies and the world model are trained over 6000 episodes. In addition, the values in the policies are initialized with a bias to reflect the velocity preference towards the goal in the reward structure. For simplicity, the world model is simplified into a Markov decision process with surprise measured from the state space. Unless otherwise stated, we set the parameters $\Omega = 0.425$, $\epsilon = 0.2$, $\gamma = 0.5$, and $\varsigma=0.05$.
	\begin{figure}[t!]
		\begin{minipage}{0.49\linewidth}
			\centering
			\includegraphics[scale=0.55]{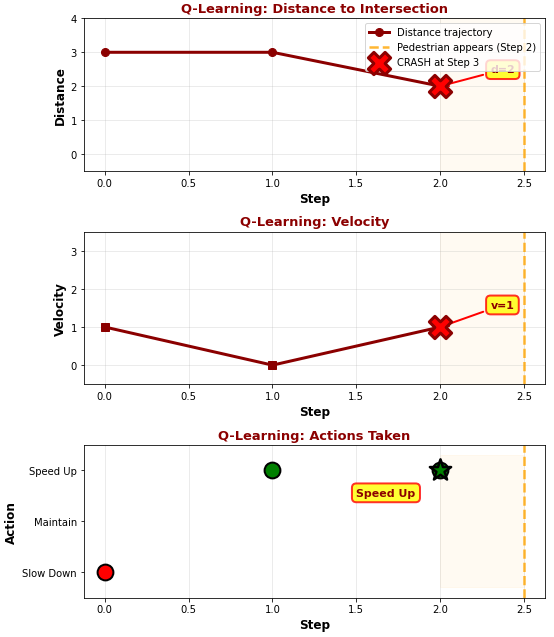}
			\subcaption{}    \label{comparison_Q}
		\end{minipage}
		\begin{minipage}{0.49\linewidth}
			\centering
			\includegraphics[scale=0.55]{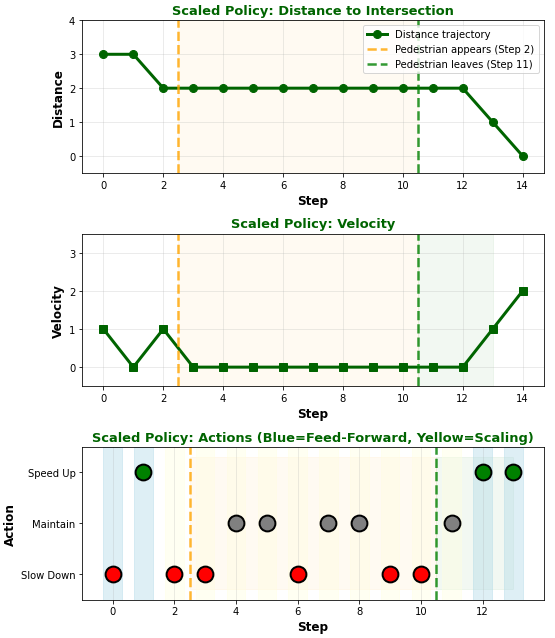}
			\subcaption{}    \label{comparison_Scaling}
		\end{minipage}
		\caption{\small{Comparison between (a) Q-learning and (b) the proposed test-time scaling on an unforeseen jaywalking scenario}. In (a), the Q-learning agent crashes at Step $2$ upon the pedestrian's appearance, failing to adapt its policy. In (b), the proposed scaling law detects the pedestrian as an unforeseen scenario at Step $2$, triggering a switch from feed-forward (blue) to scaled (yellow) actions, allowing the agent to safely navigate the scenario and resume its original policy once the pedestrian clears at Step $11$.}  \label{comparison_Q_Scaling}
		\vspace{-0.3cm}
	\end{figure}

Fig.~\ref{comparison_Q_Scaling} evaluates the performance of the baseline Q-learning policy and the proposed scaled policy in an unforeseen scenario.
As shown in Fig.~\ref{comparison_Q}, the vehicle starts its episode at $d=3$ and initial velocity $v=1$.
As the agent moves forward, an unforeseen scenario appears at step~2. In particular, the pedestrian appears when the agent is approximately at $d=2$ from the goal and traveling at velocity $v=1$. For Q-learning, the agent continues to accelerate even with the unforeseen scenario. Typically, this reflects the overconfident generalization whereby the agent's learned association between green traffic lights and safe acceleration, reinforced across training episodes, dominates in this new pedestrian situation. Eventually, this leads to a collision at step~$3$ and episode termination. As a result, Q-learning fails to complete the episode under this unseen condition. 
Similarly, the pedestrian is encountered at the same step and distance in the case of the proposed, scaled policy, as shown in Fig.~\ref{comparison_Scaling}. Nevertheless, the agent detects the unforeseen scenario. Accordingly, the agent reasons to minimize its \ac{VFE}. In contrast to the Q-learning solution, the agent scales its policy to slow down its velocity to $v=0$ . The agent then remains static for roughly eight steps (steps~$3$--$11$) while the pedestrian crosses to avoid any collision (which can increase \ac{VFE}). Once the pedestrian clears at step~11, the agent resumes acceleration and reaches the destination at step~14. 
Overall, the agent completes the task with no collisions in case of test-time scaling, whereas Q-learning fails. This showcases the ability of the scaled policy to achieve zero-shot generalization and adapt safely at test-time, in contrast to the brittle behavior of standard Q-learning. 
Most importantly, this improvement stems not from retraining or additional data, but from surprise-triggered reasoning that temporarily overrides the habitual base policy $\pi_o$ upon detecting the out-of-distribution pedestrian. Thus, this demonstrates that test-time scaling alone is sufficient to handle unforeseen scenarios.

\begin{figure}[t!]
	\begin{minipage}{\linewidth}
		\centering
		\includegraphics[scale=0.95]{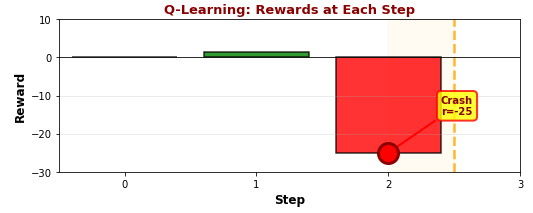}
		\subcaption{} \label{Rewards_Q}
	\end{minipage}

	\vspace{0.3cm}

	\begin{minipage}{\linewidth}
		\centering
		\includegraphics[scale=0.92]{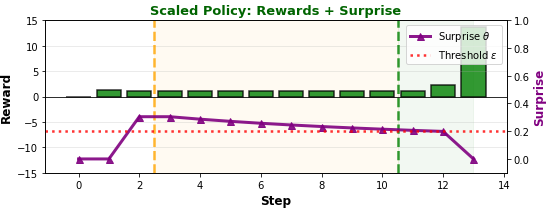}
		\subcaption{} \label{Rewards_Surprise_Scaling}
	\end{minipage}

	\caption{\small{Comparison between the rewards of (a) Q-learning and (b) the proposed test-time scaling over the time steps in the unforeseen scenario, with the variation of surprise levels during policy scaling in this scenario}. }
	\label{Comparison_Rewards}
	\vspace{-0.3cm}
\end{figure}

Fig.~\ref{Comparison_Rewards} shows how the rewards of Q-learning and scaled policy progress over time steps $\tau$. In particular, both solutions share the same rewards across the first two steps. However, as the Q-learning policy fails to generalize in response to the unforeseen scenario, it results in collision with the pedestrian incurring a penalty $r = -25$ at $\tau =2$. This is illustrated in Fig.~\ref{Rewards_Q}. This is due to the fact that the policy is dominated by its task reward. Effectively, this terminates the episode. In contrast, Fig.~\ref{Rewards_Surprise_Scaling} shows that surprise $\theta$ in our proposed method remains below the threshold $\epsilon$ for $\tau < 2$, as predicted perceived states match the transitions distribution acquired during training\footnote{In this work, $\epsilon$ is treated as a fixed threshold for simplicity. However, in general it may not need to be constant, as neuroscience suggests that the threshold for triggering deliberative reasoning is modulated by arousal, attention, and prior experience~\cite{bogacz2010neural}. This leaves the possibility of a learned or dynamically adapted $\epsilon$ as future work.}. However, as the pedestrian appears at $\tau=2$ having $L=1$, surprise abruptly increases to reach $\theta = 0.29$ and crosses the threshold $\epsilon$. This signals an unforeseen scenario has occurred and for which the agent must reason to generalize. Here, reasoning about the \ac{EFE} penalizes actions leading to non-preferred states (e.g., collision with pedestrian). As shown in Fig.~\ref{comparison_Q_Scaling}, the agent selects to slow down, reducing velocity from $v=1$ to $v=0$. The agent maintains $v=0$ for $3\leq \tau \leq 12$ while surprise remains above the threshold while the pedestrian crosses. At $t=12$, the pedestrian clears and surprise drops to $\theta=0.19 < \epsilon$. Hence, the agent automatically reverts to its Q-learning policy without inference. Subsequently, the agent then accelerates to reach its goal at $\tau=14$ with cumulative reward $r = 22.9$. Clearly, this shows how our proposed test-time scaling method generalizes to unforeseen scenarios while outperforming Q-learning in terms of cumulative rewards that guarantee successfully finishing their episode. This fosters the premise of \ac{RL} which essentially is to maximize the long-term rewards rather than the instantaneous rewards.

\begin{figure}[t!]
	\centering
	\includegraphics[width=0.95\textwidth]{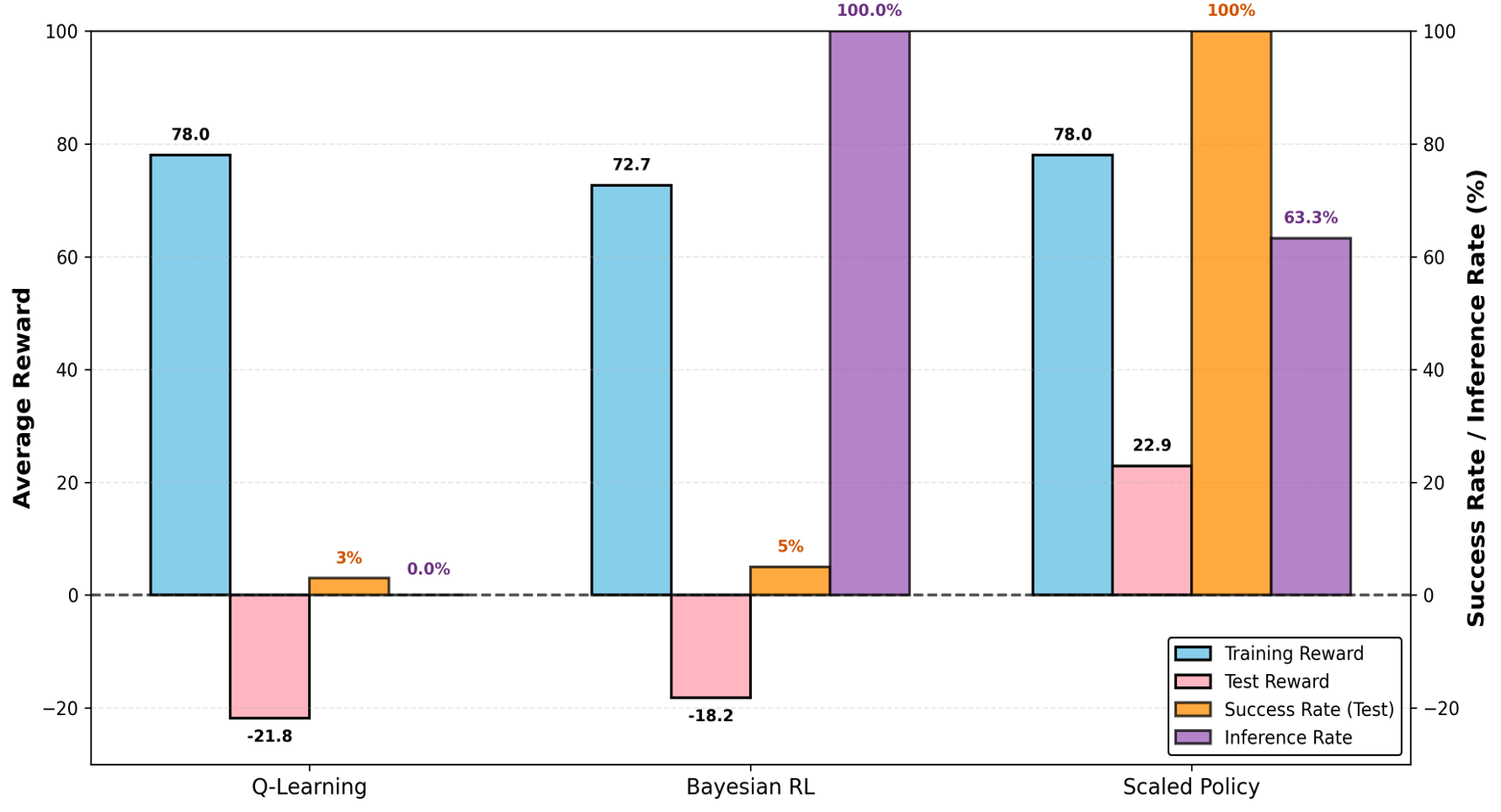}
	\caption{\small{Performance comparison in terms of training rewards, test-time rewards (in unforeseen scenario), success rate, and inference rate between Q-learning, Bayesian \ac{RL}, and the proposed test-time scaling method.}}
	\label{Comparison_All}
		\vspace{-0.4cm}
\end{figure}

Fig.~\ref{Comparison_All} presents a comprehensive performance comparison across three methods: Q-learning, Bayesian \ac{RL}, and our proposed test-time policy scaling approach. The analysis examines training reward, test-time reward, success rate on out-of-distribution scenarios, and inference rate indicating computational efficiency. All three methods achieve comparable training performance, with Q-learning and the proposed scaled policy obtaining average rewards of $78.0$, while Bayesian \ac{RL} achieves $72.7$. On the one hand, this similar performance is due to fact that the scaled policy and Q-learning method utilize the same objective during training, with the world model learned jointly without modifying the reward optimization process. On the other hand, the slight degradation in Bayesian \ac{RL} training performance ($6.8$\% lower than Q-Learning) reflects the computational overhead of maintaining uncertainty estimates during exploration, which constrains policy optimization efficiency.

Nevertheless, Q-Learning and Bayesian \ac{RL} still largely fail to adapt to the unforeseen scenario at test-time, as shown in Fig.~\ref{Comparison_All}. This results in an average cumulative reward of $-21.8$ and $3$\% success rate for Q-learning, as it lacks the necessary experience in dealing with this new situation. However, Bayesian \ac{RL} demonstrates marginal improvement with test reward of $-18.2$ and $5$\% success rate, indicating that maintaining uncertainty estimates alone provides insufficient robustness without explicitly acquiring surprise-based inference.
In contrast, the proposed test-time scaling method outperforms both Q-learning and Bayesian \ac{RL} by achieving a reward $r = 22.9$ with 100\% success rate. This performance gain directly stems from the surprise (i.e., \ac{VFE}) detection mechanism that automatically triggers active inference and reasoning about minimizing the \ac{EFE} when encountering unforeseen transitions. This preserves the agent from blindly following the policy in unforeseen scenarios, whereby the agent can survive in these instances by \emph{sacrificing its narrow task rewards} to resolve its uncertainty about the world. This is clearly reflected in the drop of rewards in the scaled policy method to 22.9 at test time while achieving a $100\%$ success rate over episodes within the evaluated environment and setup, which demonstrates how world models combined with inference-based compute scaling enable robust generalization without retraining.
That said, the proposed scaling method requires only $63.3\%$ inference rate in comparison to Bayesian \ac{RL}. Indeed, this corresponds to the 8 timestep window (steps 3--10 out of 14 total) where surprise exceeds threshold $\epsilon$ during the jaywalking scenario and requires reasoning therein. Hence, test-time scaling is able to achieve a computational tradeoff that balances between efficient, model-free Q-learning (i.e., no inference overhead) and model-based Bayesian \ac{RL}. Hence, the proposed scaling method efficiently utilizes inference compute with over $36\%$ enhancement in contrast to Bayesian \ac{RL} which completely depends on inference. 
This efficiency gain reflects our central contribution, i.e., by scaling inference compute only when surprise is detected, the proposed method allocates deliberative reasoning precisely when it is needed, avoiding the constant computational overhead of full Bayesian \ac{RL} while retaining generalization capabilities.
Thereby, these results demonstrate that adaptive compute scaling enables safe navigation of unforeseen scenarios, while preserving efficiency on familiar states by limiting computational overhead.

\begin{figure}[t!]
	\centering
	\begin{minipage}{0.49\linewidth}
		\includegraphics[scale=0.53]{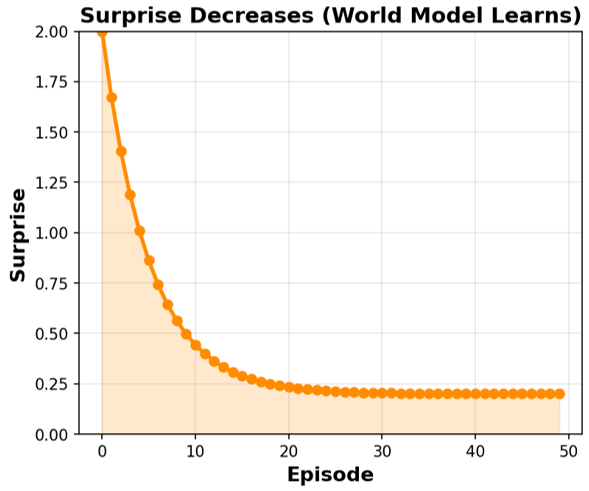}
		\subcaption{} \label{Update_Surprise}
	\end{minipage}
	\hfill
	\begin{minipage}{0.49\linewidth}
		\includegraphics[scale=0.53]{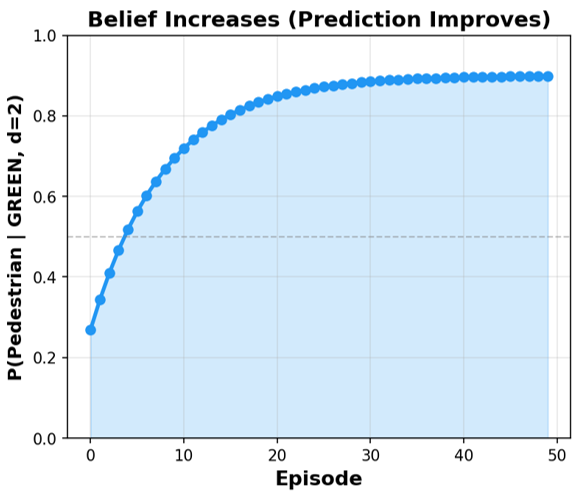}
		\subcaption{} \label{Update_PedProb}
	\end{minipage}
	\vfill
	\begin{minipage}{0.96\linewidth}
		\centering
		\includegraphics[scale=0.53]{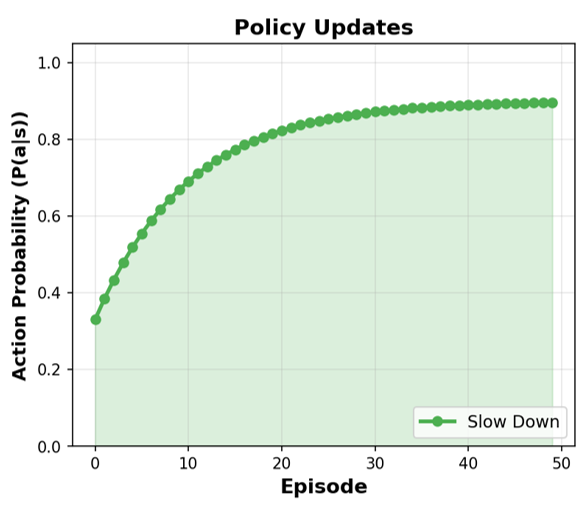}
		\subcaption{} \label{Update_Policy}
	\end{minipage}

	\caption{\small{(a) The change in surprise and update of the (b) world model and (c) policy parameters through inference while encountering the resolved unforeseen scenarios over 50 episodes.}}
	\label{learning_worldmodel_policy}
	\vspace{-0.5cm}
\end{figure}

Fig.~\ref{learning_worldmodel_policy} illustrates the learning and model update dynamics of the test-time scaling approach after encountering the same unforeseen scenario over 50 episodes. 
In Fig.~\ref{Update_Surprise}, the agent initially exhibits a high (non-normalized) surprise of $\mathscr{S} = 1.95$ bits upon first encountering the unforeseen jaywalking scenario at episode 0. Through repeated exposure, the world model beliefs update such that the unforeseen scenario becomes progressively less surprising over episodes. Asymptotically, the surprise $\mathscr{S}$ approaches $\mathscr{S} \approx 0.2$ bits by episode 50. In this case, the agent confidently expects a pedestrian to appear in the situation. Indeed, Fig.~\ref{Update_PedProb} demonstrates the corresponding improvement in pedestrian prediction accuracy, whereby the pedestrian prediction accuracy at green lights $P(\text{pedestrian}=1| L=1, d=2)$ increases from $0.25$ at episode 0 to $0.9$ as it approaches episode 50. Effectively, this indicates successful adaptation of the world model's transition dynamics as the physical world changes in non-stationary conditions.

As the probability of this situation becomes highly expected (i.e., less surprising), the agent must also update to incorporate this situation into its policy. Here, Fig.~\ref{Update_Policy} shows this policy update, where the action probability to slow down increases from $0.3$ to $0.92$ over 50 episodes. This exemplifies the transition of action control from reasoning-based decisions to feed-forward execution as the unforeseen scenario becomes familiar with repeated experience. Henceforth, this test-time scaling approach offers computational efficiency as the unforeseen scenario becomes integrated into the world model, eliminating the need for expensive reasoning, once beliefs converge below surprise threshold. In addition, it enables a \emph{continual learning} solution where the \ac{AI} agent can continue to confidently learn and improve with more experience directly from its world, while limiting risks of exploration and uncertainty.
These findings empirically validate the theoretical contributions introduced in Section~\ref{Learning via Inference}, confirming that learning to minimize prediction errors translates into measurable gains in terms of adaptability, generalization, and computational efficiency at test time.

\section{Conclusion}

In this paper, we have derived a test-time scaling law for physical \ac{AI} agents grounded in the first principle of active inference. By equipping physical \ac{AI} agents with active inference capabilities, these agents become endowed with a general objective to survive in line with their narrow task objectives. Thus, we have shown how this urge to survive can drive physical \ac{AI} agents to generalize at test time through reasoning with their world model. In particular, this reasoning aims to control the surprise of the \ac{AI} agent through resolving prediction errors between the world model and the real world. Thus, we have introduced a surprise per action-value that scales the policy at test time by updating it with this reasoning. We have modeled this update as a soft Bayesian inference with the reasoning being a likelihood.  
This update was shown to capture the biological mechanism that scales the policy in the \ac{BG} with reasoning from the \ac{PFC} at test time. Then, we have introduced a variational inference solution to render the inference problem tractable through free energy minimization. Through an equivalent  statistical physics treatment, we have modeled the process of resolving prediction errors as a gradient descent on the free energy landscape to formalize the scaled posterior policy. Furthermore, we have shown how our framework enables physical \ac{AI} agents to autonomously scale their experience while continuously interacting with the world. In particular, we have showcased how our solution extends learning beyond training by reinforcing new instances, resolved at test time, into both the policy and world model. Simulation results in an autonomous driving task have validated the proposed framework, by outperforming model-free Q-learning and model-based Bayesian \ac{RL}, through achieving robust generalization to unforeseen scenarios while improving inference efficiency by over $36\%$.

\appendices 

\section{Proof of Lemma 1}

    Starting from the non-negativity of the \ac{KL} divergence, we have:
    \begin{align}
    D_{KL}\Big[ q(s_{1:T}, \pi_o) \,\|\, p(s_{1:T}, \pi_o \mid o_{1:t}) \Big] \geq 0 \implies  \mathbb{E}_{q(s_{1:T},  \pi_o)} \Big[ \ln q(s_{1:T}, \pi_o) - \ln p(s_{1:T}, \pi_o \mid o_{1:t}) \Big] \geq 0.
    \label{non-negativity}
    \end{align}
    Using Bayes’ rule, we can expand the posterior:
    \begin{equation}
    \ln p(s_{1:T}, \pi_o \mid o_{1:t})
    = \ln p(o_{1:t}, s_{1:T}, \pi_o) - \ln p(o_{1:t} \mid \pi_o).
    \label{expansion}
    \end{equation}
    Substituting~\eqref{expansion} back in~\eqref{non-negativity} and rearranging the terms, we obtain:
    \begin{equation*}
    \underbrace{\mathbb{E}_{q(s_{1:T}, \pi_o)} \Big[ \ln q(s_{1:T}, \pi_o) - \ln p(o_{1:t}, s_{1:T}, \pi_o) \Big]}_{F[q(s_{1:T}, \pi_o)]} 
    + \ln p(o_{1:t} \mid \pi_o) \geq 0.
    \end{equation*}
    Thus, we reach the standard negative ELBO equation:$ - \ln p(o_{1:t} \mid \pi_o) \leq F[q(s_{1:T}, \pi_o)].$ Considering that the \ac{AI} agent experiences an unforeseen scenario at time $t$ whereby surprise arises, we can simplify the cumulative surprise using a mean field approximation, as follows:
    $$- \ln p(o_{1:t} \mid \pi_o) \approx -\sum_{\tau=1}^t \ln p(o_\tau \mid \pi_o) = \underbrace{-\sum_{\tau=1}^{t-1} \ln p(o_\tau \mid \pi_o)}_{\approx 0} -\ln p(o_t \mid \pi_o)  = -\ln p(o_t \mid \pi_o) = \mathscr{S} (t, \pi_o)$$ 
    Therefore, we conclude that: 
    $ \mathscr{S} (t, \pi_o) \leq F[q(s_{1:T}, \pi_o)].$ Thus, the lemma is proved.

\section{Proof of Proposition 1}
We start from the per-step \ac{EFE} for policy \(\pi_o\):
\begin{equation} 
G(\tau, \pi_o)
= \mathbb{E}_{q(o_\tau,s_\tau\mid\pi_o)}\!\big[ \ln q(s_\tau\mid\pi_o) - \ln p(o_\tau,s_\tau\mid\pi_o) \big].
\end{equation}
Here, the expectation is taken under $q (\cdot \mid \pi_o)$ considering both future outcomes and states given policy \(\pi_o\).
Expanding the joint density using \(p(o_\tau,s_\tau\mid\pi_o)=p(s_\tau\mid o_\tau,\pi_o)\,p(o_\tau\mid\pi_o)\) and separating the outcome prior (preferences) yields:
\begin{align} 
G(\tau, \pi_o)
&= \mathbb{E}_{q(o_\tau,s_\tau\mid\pi_o)}\!\big[ \ln q(s_\tau\mid\pi_o) - \ln p(s_\tau\mid o_\tau,\pi_o) \big] 
    \;-\; \mathbb{E}_{q(o_\tau,s_\tau\mid\pi_o)}\!\big[ \ln p(o_\tau \mid \pi_o) \big].
\end{align}

A common practice when planning in active inference is to approximate the true posterior \(p(s_\tau\mid o_\tau,\pi_o)\) by the variational posterior \(q(s_\tau\mid o_\tau,\pi_o)\). Another modification is expressing the preference prior explicit as \(p(o_\tau\mid C)\), as the model evidence depends on the preferences $C$ encoded within $\mathcal{W}$. Making these replacements (which is the standard approximation used to expose epistemic terms) gives:
\begin{align} 
G(\tau, \pi_o)
&\approx \mathbb{E}_{q(o_\tau,s_\tau\mid\pi_o)}\!\big[ \ln q(s_\tau\mid\pi_o) - \ln q(s_\tau\mid o_\tau,\pi_o) \big]
    \;-\; \mathbb{E}_{q(o_\tau,s_\tau\mid\pi_o)}\!\big[ \ln p(o_\tau\mid C) \big].
    \label{expanded}
\end{align}

After expanding expectations into sums over outcomes and states and using \(q(o_\tau,s_\tau\mid\pi_o)=q(s_\tau\mid\pi_o)\,p(o_\tau\mid s_\tau)\), we rearrange~\eqref{expanded} to bring in \(\ln q(o_\tau\mid\pi_o)\) and obtain:
\begin{align} 
G(\tau, \pi_o)
&= \sum_{o_\tau,s_\tau} p(o_\tau\mid s_\tau)\, q(s_\tau\mid\pi_o)\,
     \Big( \ln \frac{q(s_\tau\mid\pi_o)}{ q(s_\tau\mid o_\tau,\pi_o)} \Big)
     \;-\sum_{o_\tau,s_\tau} p(o_\tau\mid s_\tau)\, q(s_\tau\mid\pi_o)\, \ln p(o_\tau\mid C). \nonumber\\
&= \sum_{o_\tau,s_\tau} p(o_\tau\mid s_\tau)\, q(s_\tau\mid\pi_o)\,
     \Big( \ln \frac{q(o_\tau,s_\tau\mid\pi_o)}{ q(s_\tau\mid o_\tau,\pi_o) p(o_\tau\mid s_\tau)} \Big)
     \;-\sum_{o_\tau,s_\tau} p(o_\tau\mid s_\tau)\, q(s_\tau\mid\pi_o)\, \ln p(o_\tau\mid C). \nonumber\\
 &= \sum_{o_\tau,s_\tau} p(o_\tau\mid s_\tau)\, q(s_\tau\mid\pi_o)\,
     \Big( \ln \frac{q(o_\tau\mid\pi_o)}{p(o_\tau\mid s_\tau)} \Big)
    \;-\sum_{o_\tau,s_\tau} p(o_\tau\mid s_\tau)\, q(s_\tau\mid\pi_o)\, \ln p(o_\tau\mid C). \nonumber\\
 &= \sum_{o_\tau,s_\tau} p(o_\tau\mid s_\tau)\, q(s_\tau\mid\pi_o)\,
     \Big( \ln \frac{q(o_\tau\mid\pi_o)}{p(o_\tau\mid C)} \Big)
    \;-\sum_{o_\tau,s_\tau} p(o_\tau\mid s_\tau)\, q(s_\tau\mid\pi_o)\, \ln  p(o_\tau\mid s_\tau). \nonumber
\end{align} 
Then, we marginalize $s_{\tau}$ from the first term such that  \(  \sum_{s_\tau}q(s_\tau\mid\pi_o)\,p(o_\tau\mid s_\tau) = \sum_{s_\tau}q(o_\tau,s_\tau\mid\pi_o)= q(o_\tau\mid\pi_o)\) and separate the sum in the second term to factorize in terms of the the conditional entropy such that $ H\!\big[ p(o_\tau\mid s_\tau)\big] = -  \sum_{o_\tau} p(o_\tau\mid s_\tau) \ln [p(o_\tau\mid s_\tau)]$ to reach:
\begin{align} 
G(\tau, \pi_o)
&= \sum_{o_\tau} q(o_\tau\mid\pi_o)\, \ln \frac{q(o_\tau\mid\pi_o)}{p(o_\tau\mid C)}\;+\; \sum_{s_\tau} q(s_\tau\mid\pi_o)\, H\!\big[ p(o_\tau\mid s_\tau) \big] \nonumber\\
&= \underbrace{D_{\mathrm{KL}}\!\big[ q(o_\tau\mid\pi_o) \;\|\; p(o_\tau\mid C) \big]}_{risk}
    \;+\; \underbrace{\mathbb{E}_{q(s_\tau\mid\pi_o)}\!\big[ H\!\big( p(o_\tau\mid s_\tau) \big) \big]}_{ambiguity} \nonumber.
\end{align}

\section{Proof of Corollary 1}
Following a similar derivation as in~\eqref{MFA_0}--\eqref{Euler-Lagrange} to find the optimal posterior for states $s_{1:t}$, the optimal posterior for policies $\pi \in \Pi$ can be reached to be:
\begin{equation}
\ln q^{*}(\pi) \;\propto\; \ln p(\pi) + \mathbb{E}_{q(s_{1:T})}\!\big[\ln p(o_{1:T},s_{1:T} \mid \pi)\big].
\label{ln_q*}
\end{equation}
Furthermore, splitting the time horizon at current time instant $t$ and factorizing the joint distribution gives:
\begin{align}
\mathbb{E}_{q(s_{1:T})}\!\big[\ln p(o_{1:T},s_{1:T} \mid \pi)\big]
&=  \mathbb{E}_{q(s_{1:T})}\!\big[ \ln p(o_{1:t},s_{1:t}\mid\pi)\big] + \mathbb{E}_{q(s_{1:T})}\! \big[\ln p(o_{t+1:T},s_{t+1:T}\mid s_{1:t},\pi)  \big]   
\label{Time_Horizon}
\end{align}

Then, we simplify the expectation in the first term as it only depends on $s_{1:t}$. 
Furthermore, averaging over past time instants before $t$ will marginalize $s_{1:t}$ from the second term. Similar to~\eqref{EFE}, we treat $o_{t+1:T}$ as random variables and that can be introduced into the expectation. Thus, \eqref{Time_Horizon} is transformed into: 
\begin{align}
\mathbb{E}_{q(s_{1:T})}\!\big[\ln p(o_{1:T},s_{1:T} \mid \pi)\big]
&= \underbrace{\mathbb{E}_{q(s_{1:t})}[\ln p(o_{1:t},s_{1:t}\mid\pi)]}_{\textrm{past + present}}
+ \underbrace{\mathbb{E}_{q(s_{t+1:T},o_{t+1:T})}[\ln p(o_{t+1:T},s_{t+1:T}\mid\pi)]}_{\textrm{future}}.
\label{Time_Horizon_2}
\end{align}
Here, the first term in~\eqref{Time_Horizon_2} depends on past instants $\tau < t$ when the \ac{AI} agent is committed to policy $\pi_o$ after convergence during training. In addition, we recall that the \ac{VFE} is considered equal for all policies at current time $t$. Hence, we can safely drop the first term in~\eqref{Time_Horizon_2} in the comparison between policies.
Moreover, we can rearrange the \ac{EFE} in~\eqref{EFE} to express the second term in~\eqref{Time_Horizon_2} in terms of the time instants $\tau > t$ as:
\begin{equation}
    \mathbb{E}_{q(s_{t+1:T},o_{t+1:T})}[\ln p(o_{t+1:T},s_{t+1:T}\mid\pi)] = -G(\pi) + \mathbb{E}_{q(s_{t+1:T})}[\ln q(s_{t+1:T})].
    \label{second-term}
\end{equation}
Here, $q(s_{t+1:T})$ is fixed when performing the coordinate-ascent update for $q(\pi)$.
Accordingly, $\mathbb{E}_{q(s_{t+1:T})}[\ln q(s_{t+1:T})]$ is considered to be a constant term as it does not depend on $\pi$ and is later absorbed in the normalization of $q(\pi)$. Therefore, after substituting~\eqref{second-term} back in~\eqref{ln_q*}, we can reach that: 
$\ln q^{*}(\pi) \;\propto\; \ln p(\pi) - G(\pi).$ While noting that the \ac{AI} agent has no prior preferences over the policies $\pi \in \Pi$ nor develops any habitual behavior which can favor an alternative policy over another, then we can assume that $p(\pi)$ is uniform i.e., all policies are equally likely a priori. In this case, we can simplify it further to become $\ln q^{*}(\pi) \;\propto\;  - G(\pi).$

\section{Proof of Lemma 2}

Consider the Fokker--Planck equation governing the density $p(x,t)$ under drift $f(x)$ and diffusion $\Gamma$:
\begin{equation*}
\frac{\partial  p(x,t)}{\partial t} = -\nabla \cdot \left[f(x)p(x,t)\right] + \nabla \cdot \Gamma \nabla p(x,t).
\end{equation*}
At steady state, we have $\dfrac{\partial p(x,t)}{\partial t}  = 0$. This yields the divergence:
\begin{equation}
\nabla \cdot \big[\underbrace{f(x)p^*(x) - \Gamma \nabla p^*(x)}_{J(x)}\big] = 0,
\label{divergence}
\end{equation}
where $J(x) = f(x)p^*(x) - \Gamma \nabla p^*(x)$ is the stationary flux (i.e., probability current) at \ac{NESS}. Dividing $J(x)$ by $ p^*(x)$ and rearranging to isolate $f(x)$ gives:
\begin{equation}
f(x) = \frac{J(x)}{p^*(x)} + \Gamma \nabla \ln p^*(x).
\label{normalized flow}
\end{equation}
Considering both~\eqref{divergence}~and~\eqref{normalized flow}, it is clear that the normalized flux $\frac{J(x)}{p^*(x)}$ is divergence free.
Then, by applying the Helmholtz decomposition, this divergence-free field can be written as a skew-symmetric operator $R$ acting on the gradient of the Lagrangian $L(x)$ of the \ac{NESS} system, i.e., 
\begin{equation}
\frac{J(x)}{p^*(x)} = R \nabla L(x),
\label{normalized flow 2}
\end{equation}
where $R^\top = -R$. By considering that $L(x) = -\ln p^*(x)$ and replacing~\eqref{normalized flow 2} back in~\eqref{normalized flow}, the decomposition reaches its final form as:
\begin{equation*}
f(x) = -\Gamma \nabla L(x) + R \nabla L(x) = (R - \Gamma) \nabla L(x) = (\Gamma - R) \nabla \ln p^{*}(x).
\end{equation*}
Thus, the lemma is proved.




\bibliographystyle{IEEEtran}
\bibliography{bibliography}
\end{document}